\documentclass[11pt,authoryear]{article}
\usepackage[margin=1in]{geometry}
\usepackage{amsmath,amssymb,amsthm}
\usepackage{algorithm}
\usepackage{algpseudocode}
\usepackage{graphicx}
\usepackage[hidelinks,colorlinks=true,linkcolor=blue,citecolor=blue]{hyperref}

\newcommand{\St}{\mathrm{St}}
\newcommand{\RetrNS}{\operatorname{NSOrth}}
\newcommand{\Polar}{\operatorname{Polar}}

\newtheorem{definition}{Definition}
\newtheorem{assumption}{Assumption}
\newtheorem{lemma}{Lemma}
\newtheorem{theorem}{Theorem}
\newtheorem{proposition}{Proposition}
\newtheorem{remark}{Remark}

 \usepackage{natbib}
\title{Newton--Schulz Retraction-Based Inference Enables Hidden Quantum Markov Models to Outperform Classical HMMs}
\author{Ning Ning\\\\
Department of Statistics, Texas A\&M University\\
College Station, Texas, USA\\
\texttt{patning@tamu.edu}}
\date{}

\begin{document}

\maketitle

\begin{abstract}
Hidden Markov models (HMMs) are widely used probabilistic models for discrete sequential data but can be limited when hidden dynamics are complex. Hidden quantum Markov models (HQMMs) generalize HMMs by replacing probability vectors with density matrices and stochastic transitions with quantum operations, enabling richer latent representations. However, existing HQMM learning methods have not consistently outperformed Expectation--Maximization (EM)-trained HMMs on data not generated by quantum processes, limiting their practical applicability. We introduce NS-RIS, Newton--Schulz Retraction-based Inference on the Stiefel manifold, a scalable algorithm for learning trace-preserving HQMMs. NS-RIS uses Newton--Schulz orthogonalization to compute a polar-factor search direction while preserving Stiefel-manifold feasibility, avoiding costly matrix decompositions. We establish a finite-time stationarity guarantee under standard assumptions on smoothness, stochastic gradients, and finite Newton--Schulz accuracy. Importantly, NS-RIS is the first HQMM inference algorithm with mathematical performance guarantee. Empirically, NS-RIS provides the first benchmark evidence that an HQMM can significantly outperform an EM-trained HMM on data not generated by a quantum model. On synthetic HMM-generated benchmarks, NS-RIS outperforms both EM and the state-of-the-art HQMM method COSM, improving the evaluation metric by an average of $38.5\%$ and by up to $50.6\%$. On a synthetic HQMM benchmark, it improves the test metric over COSM by $18.9\%$ while reducing runtime by $12.0\%$. On the real-world Splice classification benchmark, NS-RIS also surpasses both EM and COSM in higher-dimensional latent regimes, reducing mean classification error by $17.9\%$ for latent dimension $6$ and $14.9\%$ for latent dimension $8$ relative to COSM. These results move HQMMs beyond a theoretical generalization of HMMs and establish them as practical and expressive models for scientific sequence data.

\end{abstract}

\tableofcontents

\section{Introduction}
Hidden Markov models (HMMs) are a standard language for sequential data.
They are used wherever an observed time series is driven by an unobserved
state process, including computational biology, speech and language
processing, signal analysis, finance, and many other scientific areas. In
biology alone, HMMs and profile HMMs have become foundational tools for
modeling protein families, sequence motifs, and splice-site structure
\citep{Krogh1994,Eddy1998,Burge1998}. Their success comes from a simple and
interpretable architecture: a hidden Markov chain evolves over time and
emits observations conditionally on its current state. This simplicity also
creates a limitation. When the hidden mechanism contains higher-order,
contextual, or nonclassical dependencies, a probability vector over a finite
set of latent states may require a large state space or may fail to represent
the relevant dependence efficiently.

Hidden quantum Markov models (HQMMs) provide a principled extension of this
classical framework. More broadly, researchers in physics and machine
learning have developed quantum graphical models  by incorporating the
quantum mechanical view of probability into graphical-model inference
\citep{warmuth2006bayesian,leifer2008quantum,yeang2010probabilistic,leifer2013towards}.
Instead of representing the latent belief state by a probability vector, an
HQMM represents it by a density matrix; instead of using nonnegative
transition-emission matrices, it uses symbol-conditioned quantum operations
represented by Kraus operators
\citep{monras2010hidden,clark2015hidden,deb2026quantum}. This formulation
preserves the sequential likelihood structure of HMMs while allowing the
latent state to encode coherence through off-diagonal matrix entries. Prior
work has shown that HQMMs can be more expressive than classical HMMs and can
represent some sequential processes more compactly
\citep{NIPS2018_8235,ning2025robust}. Recent applications and physical
formulations have further connected HQMMs to sequential analysis,
measurement-induced quantum inference, and many-body topological structure
\citep{souissi2026hidden,kim2026measurement}. General overviews of
quantum machine learning also identify learning expressive quantum
sequential models from data as an important open problem
\citep{schuld2015introduction,biamonte2016quantum}. Related quantum
graphical model perspectives connect HQMMs to inference in Hilbert space and
to operator-based models of stochastic processes
\citep{jaeger2000observable,zhu2025channel}.

The central obstacle is learning. HQMM parameters must satisfy the
trace-preserving Kraus constraint
$\sum_{y,q}K_{y,q}^{\dagger}K_{y,q}=I$, which makes the feasible parameter
space a complex Stiefel manifold after stacking the Kraus operators. Existing
methods either pay a high computational price to maintain feasibility or use
updates that do not consistently translate HQMM expressiveness into better
performance on ordinary, non-quantum-generated data. As a result, HQMMs have
remained more compelling as a theoretical generalization of HMMs than as a
practical replacement for HMMs in the broad scientific settings where HMMs
are routinely used.

This paper introduces NS-RIS, short for Newton--Schulz Retraction-Based
Inference on the Stiefel manifold, for scalable HQMM learning. The method
uses Newton--Schulz orthogonalization twice: first to approximate the polar
factor of the Riemannian momentum direction, and second to retract the
updated Kraus matrix back to the Stiefel manifold. In this way, NS-RIS
preserves the physical trace-preserving constraint while avoiding expensive
matrix decompositions inside the repeated learning loop.  We
compare NS-RIS with the main existing HQMM learning baselines: Givens
Search (GS) \citep{srinivasan2018}, constrained optimization on the Stiefel
manifold (COSM) \citep{adhikary2020expressiveness}, and the classical
Expectation--Maximization (EM) procedure for HMMs which serves as the
standard non-quantum benchmark.

The key findings are as follows.
\begin{itemize}
    \item NS-RIS gives a decomposition-free learning method on the Stiefel
    manifold for trace-preserving HQMMs. Algorithm~\ref{alg:learn-qgm4}
    gives the full incremental learning procedure and
    Algorithm~\ref{alg:ns-orth} gives the Newton--Schulz orthogonalization
    subroutine; together they make the Kraus constraint part of the geometry
    of the algorithm rather than an after-the-fact correction.
    \item Proposition~\ref{prop:operator-steepest-retraction} mathematically
    explains the two purposes of the Newton--Schulz design. The first
    Newton--Schulz step approximates the polar factor that solves the
    operator-norm steepest-descent subproblem under constraint in the (matrix) operator norm, with optimal value obtained with the (matrix) nuclear norm; the
    second Newton--Schulz step approximates the nearest Stiefel retraction
    under the (matrix) Frobenius norm.
    \item The main theoretical contribution is
    Theorem~\ref{thm:nsris-convergence}, which establishes a finite-time
    stationarity guarantee under standard smoothness, stochastic-gradient,
    and finite Newton--Schulz accuracy assumptions. This appears to be the
    first mathematical convergence guarantee for an HQMM inference algorithm.
    The bound separates the
    descent term, smoothness term, finite Newton--Schulz feasibility
    residuals, and stochastic momentum-tracking term, showing how each source
    affects the final stationarity level.
    \item On synthetic HMM-generated data, generated from the standard
    $6$-hidden-state, $6$-output HMM benchmark, NS-RIS exceeds the
    EM-trained HMM baseline and improves over COSM by $38.5\%$ on average
    across tested latent dimensions and Kraus ranks. In the best-performing
    configuration, it reaches a $50.6\%$ relative improvement over COSM while
    maintaining competitive running time; see
    Subsection~\ref{sec:synthetic-hmm-results} and
    Figures~\ref{fig:synth-hmm-w1}--\ref{fig:synth-hmm-runtime}.
    \item On the synthetic HQMM benchmark generated by a
    $2$-hidden-state, $6$-output HQMM with Kraus rank $w=1$, NS-RIS achieves
    the best training, validation, and test metrics among GS, COSM, and EM.
    It improves the test metric over COSM by $18.9\%$ while reducing runtime
    by $12.0\%$, showing that the gain is not obtained by a larger
    computational cost; see Subsection~\ref{sec:synthetic-hqmm-results} and
    Figure~\ref{fig:synth-hqmm-runtime}.
    \item On the real Splice classification benchmark, NS-RIS outperforms EM
    and COSM in higher-dimensional latent regimes, reducing mean error by
    $17.9\%$ for latent dimension $6$ and by $14.9\%$ for latent dimension
    $8$ relative to COSM after averaging over Kraus ranks. Class-wise results
    further show improvements in EI, IE, and negative-class errors for
    $(n,w)=(6,4)$ and $(8,4)$, with the largest reductions in the
    negative-class error. The scientific validation in
    Subsection~\ref{sec:scientific-validation} explains why latent
    dimensions larger than the four-letter nucleotide alphabet are
    biologically meaningful: they can encode splice motifs, exon--intron
    regimes, positional context, and long-range dependencies; see
    Figures~\ref{fig:splice_combined_mean_error_bars}--\ref{fig:splice_runtime_nsris_cosm_bars}.
\end{itemize}
Together, these findings provide the first benchmark
evidence in this setting that an HQMM can substantially outperform an
EM-trained HMM on data not generated by a quantum model; they also establish the first HQMM inference algorithm with a mathematical convergence guarantee. This moves HQMMs
beyond the statement that they generalize HMMs mathematically: it shows that,
with an effective Stiefel-manifold learning procedure, HQMMs can reliably deliver
practical gains on ordinary scientific sequence data.

The remainder of the paper is organized as follows. Section~\ref{sec:background}
reviews HMMs and introduces HQMMs in the operator form used throughout the
paper. Section~\ref{sec:NSRIS} formulates HQMM learning on the Stiefel
manifold, presents NS-RIS, and states the main mathematical guarantees.
Section~\ref{sec:numerical-analysis} evaluates NS-RIS on synthetic HMM and
HQMM benchmarks. Section~\ref{sec:empirical-analysis} studies real splice
sequence classification and discusses its scientific implications.
Section~\ref{sec:conclusion} concludes, and Appendix~\ref{sec:proofs}
contains the proofs.

\section{From HMMs to HQMMs}
\label{sec:background}
This section recalls the classical HMM formulation in
Subsection~\ref{sec:HMM}, and then introduces the HQMM extension in
Subsection~\ref{sec:HQMM} whose quantum states and Kraus
operators generalize the hidden-state dynamics.

\subsection{Hidden Markov Models}
\label{sec:HMM}
An HMM models a discrete-time sequence by separating what is observed from
what is only indirectly inferred. At time $t$, the model has a hidden state
$X_t\in\{1,\ldots,n\}$ and emits an observation
$Y_t\in\{1,\ldots,m\}$. The hidden process is Markovian, so the distribution
of $X_t$ depends on the past only through $X_{t-1}$, while the observation
$Y_t$ is drawn conditionally on the current hidden state $X_t$. This architecture
is useful when the data exhibit temporal dependence but the mechanism
driving that dependence is not directly observed.

Let $\pi\in\mathbb{R}^n$ be the initial distribution, let
$A\in\mathbb{R}_{\geq0}^{n\times n}$ be the transition matrix with
$A_{ij}=P(X_t=i\mid X_{t-1}=j)$, and let
$C\in\mathbb{R}_{\geq0}^{m\times n}$ be the emission matrix with
$C_{yi}=P(Y_t=y\mid X_t=i)$. We use the column-stochastic convention
$\mathbf{1}^T\pi=1$, 
    $\mathbf{1}^TA=\mathbf{1}^T$, and 
    $\mathbf{1}^TC=\mathbf{1}^T$. 
If $x_{t-1}$ denotes the filtering distribution over hidden states after
the first $t-1$ observations, prediction gives $Ax_{t-1}$ before the next
symbol is observed. Once $Y_t=y_t$ is observed, the filtering distribution
is updated by
\[
    x_t
    =
    \frac{\operatorname{diag}(C(y_t,:))A x_{t-1}}
    {\mathbf{1}^T\operatorname{diag}(C(y_t,:))A x_{t-1}}.
\]
The denominator is the one-step predictive probability of $y_t$ under the
current model.

It is often convenient to combine transition and emission into a
symbol-indexed operator
\begin{equation}
    \label{eqn:T}
    T_y=\operatorname{diag}(C(y,:))A.
\end{equation}
For an observation sequence $\bar{y}=y_1,\ldots,y_T$, the likelihood is then
\[
    P(\bar{y})
    =
    \mathbf{1}^T T_{y_T}T_{y_{T-1}}\cdots T_{y_1}\pi.
\]
Thus an HMM can be viewed as a family of nonnegative linear maps
$\{T_y\}_{y=1}^m$ acting on probability vectors, with normalization after
each observation when filtering is required. This operator view is the form
that most directly connects classical HMMs to their quantum generalization:
HQMMs keep the same sequence-likelihood logic, but replace probability
vectors by density matrices and replace the nonnegative maps $T_y$ by
quantum operations.

\begin{figure}[t!]
	\centering
	\includegraphics[width=1\linewidth]{\detokenize{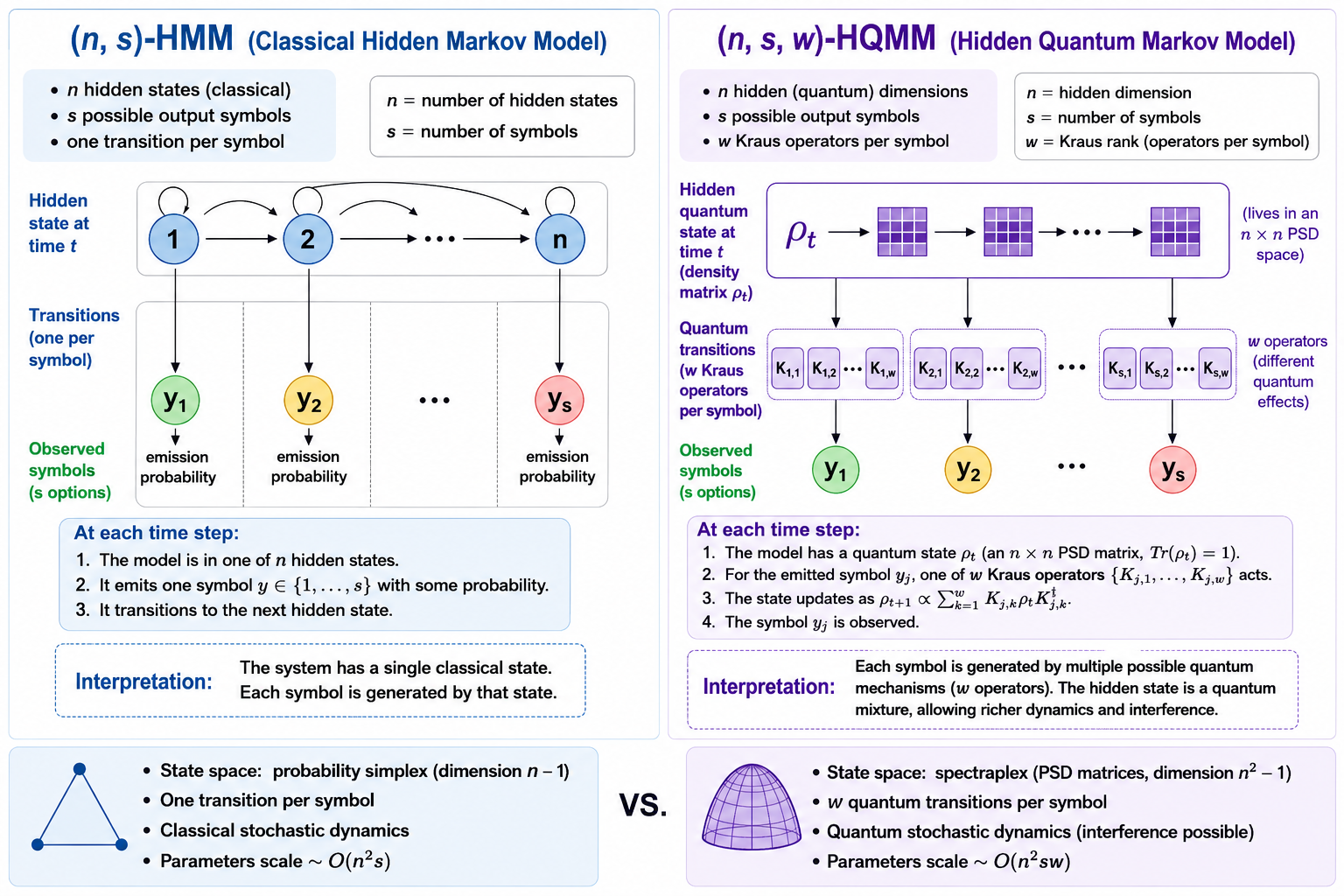}}
	\label{fig:HMM_HQMM}
\end{figure}

\subsection{Hidden Quantum Markov Models}
\label{sec:HQMM}
The passage from HMMs to HQMMs begins with the state representation. In a
classical model, the belief state is a probability vector. In a quantum
model, the state is represented by a density matrix $\rho$, a positive
semidefinite matrix with $\operatorname{tr}(\rho)=1$. The diagonal entries
of $\rho$ can be interpreted as ordinary probabilities in a chosen basis,
while the off-diagonal entries encode quantum coherence. For example, the
two-dimensional state
\[
    |\psi\rangle
    =
    \frac{1}{\sqrt{2}}|0\rangle
    -
    \frac{i}{\sqrt{2}}|1\rangle
\]
has density matrix
\[
    \rho
    =
    |\psi\rangle\langle\psi|
    =
    \begin{bmatrix}
        1/2 & i/2\\
        -i/2 & 1/2
    \end{bmatrix}.
\]
The diagonal entries assign equal probability to the two basis states, and
the off-diagonal entries retain phase information that is absent from a
classical probability vector.

The second replacement is at the level of dynamics. A quantum operation
$\mathcal{K}$ maps density matrices to density matrices and can be written
in Kraus form as
\[
    \mathcal{K}(\rho)
    =
    \sum_i K_i\rho K_i^\dagger,
\]
where $K_i^\dagger$ is the Hermitian conjugate of $K_i$. The operation is
trace-preserving when
\[
    \sum_i K_i^\dagger K_i=\mathbb{I}_n,
\]
and trace-nonincreasing when the left-hand side is bounded above by
$\mathbb{I}_n$. In an HQMM, each observable symbol is associated with a
trace-nonincreasing operation, and the sum over all symbols is
trace-preserving. This mirrors the HMM requirement that the probabilities of
all possible emissions sum to one.

\begin{definition}[\citet{monras2010hidden}]
    \label{def:HQMM}
    A hidden quantum Markov model is a quantum system with state $\rho$ and
    a set of quantum operations $\{\mathcal{K}_s\}$ indexed by output
    symbols, such that $\sum_s\mathcal{K}_s$ is trace-preserving. At each
    time step, symbol $s$ is generated with probability
    $P(s)=\operatorname{tr}[\mathcal{K}_s(\rho)]$, and the state is updated
    to $\rho_s=\mathcal{K}_s(\rho)/P(s)$.
\end{definition}

If each symbol-conditioned operation is represented by $w$ Kraus operators,
\[
    \mathcal{K}_s(\rho)
    =
    \sum_{q=1}^{w}K_{s,q}\rho K_{s,q}^\dagger,
\]
then the trace-preserving condition for the full model is
\begin{equation}
    \label{eqn:TP}
    \sum_{s=1}^{m}\sum_{q=1}^{w}K_{s,q}^\dagger K_{s,q}
    =
    \mathbb{I}_n.
\end{equation}
Given an observed sequence $y_1,\ldots,y_T$, the HQMM filtering recursion is
\begin{equation}
    \rho_t
    =
    \frac{
        \sum_{q=1}^{w}K_{y_t,q}\rho_{t-1}K_{y_t,q}^\dagger
    }{
        \operatorname{tr}\left(
        \sum_{q=1}^{w}K_{y_t,q}\rho_{t-1}K_{y_t,q}^\dagger
        \right)
    }.
    \label{eq:hqmm}
\end{equation}
The denominator is the predictive probability of the emitted symbol. Thus
HQMMs retain the forward-propagation structure of HMMs while allowing the
latent state to carry quantum coherence and allowing the transition-emission
mechanism to be modeled by completely positive maps. This additional
structure is one reason HQMMs can represent some sequential dependencies
more compactly than classical HMMs
\citep{srinivasan2018,adhikary2020expressiveness}.

\section{NS-RIS for HQMM Learning}
\label{sec:NSRIS}
This section first formulates HQMM learning on the Stiefel manifold and
describes the NS-RIS algorithm in Subsection~\ref{subsec:nsris-algorithm}.
Subsection~\ref{subsec:nsris-guarantees} then summarizes the main
mathematical assumptions and convergence guarantees; the detailed proofs are
deferred to Appendix~\ref{sec:proofs}.

\subsection{The NS-RIS Algorithm}
\label{subsec:nsris-algorithm}
Learning an HQMM requires estimating a collection of symbol-conditioned
Kraus operators that explain the observed sequences while satisfying the
trace-preserving constraint that makes the model physically admissible.
Let
$D_{\mathrm{tr}}=\{B_i\}_{i=1}^N$ denote the training set, where
$B_i=(y_{i1},\ldots,y_{iT_i})$ is a discrete observation sequence. For a
single sequence $B=(y_1,\ldots,y_T)$ and initial density matrix $\rho_0$,
the log-likelihood is obtained by propagating the unnormalized quantum state
through the Kraus operators associated with the observed symbols:
\begin{equation}
    \label{eqn:hqmm_likelihood}
    \ell(K;B)
    =
    \log
    \operatorname{tr}
    \left(
        \sum_{q_T=1}^{w}
        K_{y_T,q_T}
        \cdots
        \left(
            \sum_{q_1=1}^{w}
            K_{y_1,q_1}\rho_0 K_{y_1,q_1}^{\dagger}
        \right)
        \cdots
        K_{y_T,q_T}^{\dagger}
    \right).
\end{equation}
Here, $K$ denotes the full Kraus tensor
$\{K_{y,q}:y=1,\ldots,s,\ q=1,\ldots,w\}$, where $s$ is the output alphabet
size, $w$ is the number of Kraus operators per output, and each
$K_{y,q}\in\mathbb{C}^{n\times n}$ acts on an $n\times n$-dimensional latent quantum
state. The learning objective is the empirical negative log-likelihood
\[
    F(K)
    =
    \frac{1}{N}\sum_{i=1}^{N} L(K;B_i),
    \qquad
    L(K;B_i)=-\ell(K;B_i).
\]

Direct optimization of $F$ is challenging for two reasons. First, the
likelihood in \eqref{eqn:hqmm_likelihood} involves repeated products of
noncommuting complex matrices. Second, the Kraus operators must jointly
satisfy
\[
    \sum_{y=1}^{s}\sum_{q=1}^{w}K_{y,q}^{\dagger}K_{y,q}
    =
    \mathbb{I}_n,
\]
which is a nonlinear matrix constraint. NS-RIS handles this constraint
geometrically by stacking all Kraus operators into a single matrix
$\Gamma=(K_{1,1}^T,\ldots,K_{s,w}^T)^T\in\mathbb{C}^{swn\times n}$.
With $p=swn$, the trace-preserving condition is equivalent to
$\Gamma^\dagger \Gamma=\mathbb{I}_n$, so the feasible parameter space is the complex Stiefel manifold
\[
    \St(p,n)
    =
    \{\Gamma\in\mathbb{C}^{p\times n}:\Gamma^\dagger \Gamma=\mathbb{I}_n\}.
\]
Thus HQMM learning can be formulated as constrained optimization over $\St(p,n)$.

\subsubsection{Riemannian Incremental Update}

Let $\Gamma_k$ denote the stacked Kraus matrix at the $k$-th inner
iteration. NS-RIS selects a mini-batch index set
$\mathcal{B}_k$ according to the current random
permutation of the training set; the incremental
single-sequence case corresponds to cardinality $|\mathcal{B}_k|=1$. For mathematical convenience, we
write the sequence loss as $L(\Gamma_k;B_{i_k})$ for $i_k\in\mathcal{B}_k$. This denotes the same negative
log-likelihood $L(K_k;B_{i_k})=-\ell(K_k;B_{i_k})$ defined above, after identifying the
Kraus tensor $K_k$ with its stacked Stiefel representation
$\Gamma_k$. Thus, whenever
$L(\Gamma_k;B_{i_k})$ appears, the Kraus operators used in the likelihood
\eqref{eqn:hqmm_likelihood} are the block components of $\Gamma_k$. The
Euclidean gradient used at this iteration is
\[
    G_k
    =
    \frac{1}{|\mathcal{B}_k|}\sum_{i_k\in\mathcal{B}_k}
    \nabla_{\Gamma}L(\Gamma_k;B_{i_k})
    \in\mathbb{C}^{p\times n}.
\]
Since $\Gamma_k$ is constrained to the
complex Stiefel manifold, the update direction must lie in the tangent space
\[
    T_{\Gamma_k}\St(p,n)
    =
    \{\xi\in\mathbb{C}^{p\times n}:\Gamma_k^\dagger\xi+\xi^\dagger\Gamma_k=0\};
\]
This tangent-space equation is the first-order linearization of the Stiefel
constraint $\Gamma_k^\dagger\Gamma_k=\mathbb{I}_n$.
NS-RIS uses the orthogonal projection
\begin{equation}
\label{eqn:projection}
    \Pi_{\Gamma_k}(G_k)
    =
    G_k-\frac{1}{2}\Gamma_k\left(G_k^\dagger \Gamma_k+\Gamma_k^\dagger G_k\right),
\end{equation}
which belongs to $T_{\Gamma_k}\St(p,n)$ whenever
$\Gamma_k^\dagger\Gamma_k=\mathbb{I}_n$. Rather than forming a full
gradient over all training sequences, NS-RIS processes sequences
incrementally and maintains a momentum-smoothed tangent direction
\[
    M_{k+1}
    =
    \beta M_k+(1-\beta)\Pi_{\Gamma_k}(G_k),
    \qquad 0\leq \beta<1.
\]

\begin{figure}[t!]
	\centering
	\includegraphics[width=1\linewidth]{\detokenize{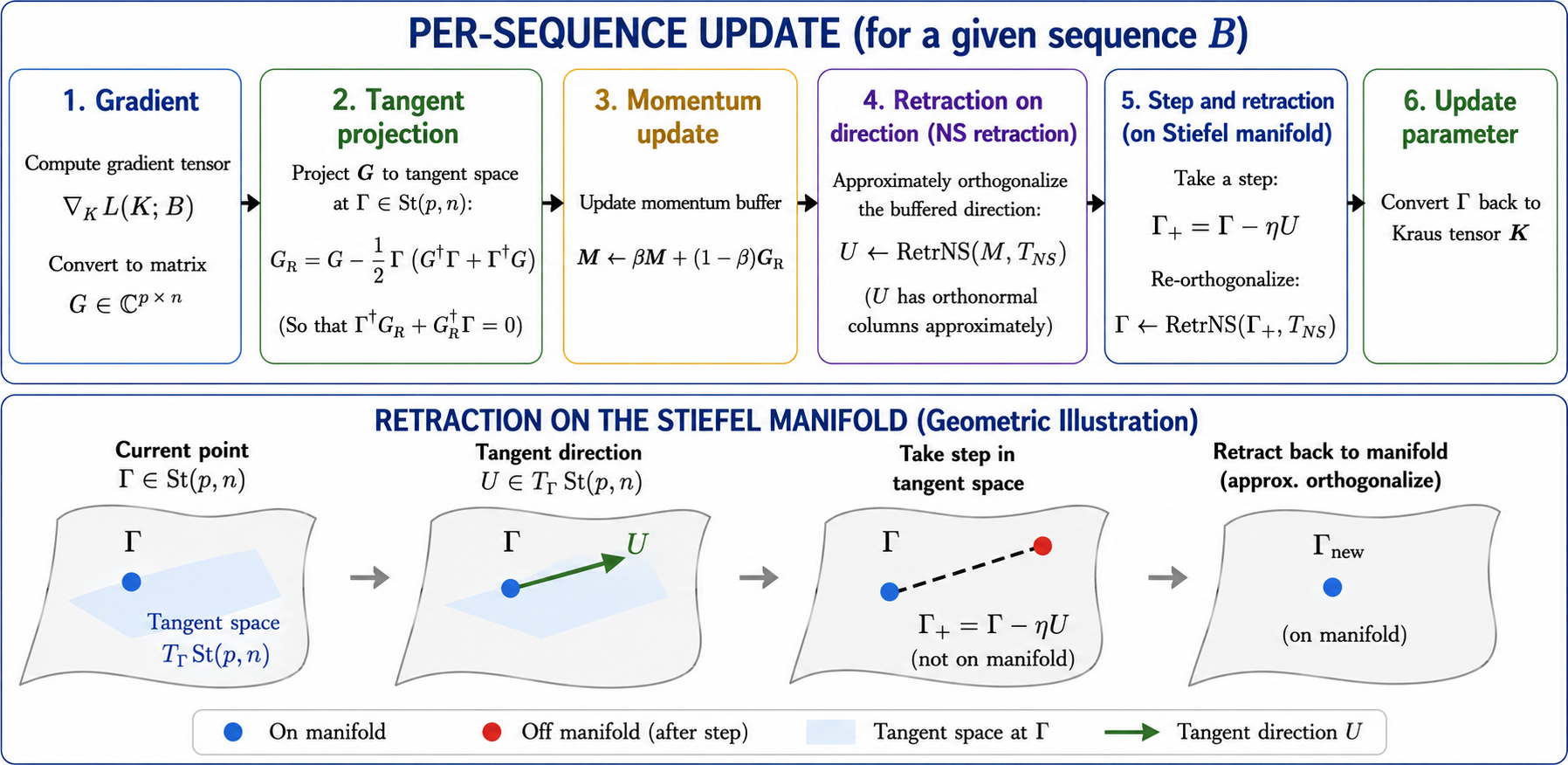}}
	%\label{fig:synth-hmm-w1}
\end{figure}

\subsubsection{Newton--Schulz Retraction}

The retraction step is implemented using a Newton--Schulz approximate
orthogonalization map, denoted by $\RetrNS(\cdot,T_{\mathrm{NS}})$, where
$T_{\mathrm{NS}}$ is the prescribed number of Newton--Schulz iterations.
Given a matrix $Z\in\mathbb{C}^{p\times n}$, the procedure first rescales
$ Z_0=\frac{Z}{\|Z\|_F+\varepsilon}$, where $\|\cdot\|_F$ denotes the
Frobenius norm,
and then applies, for $t=0,\ldots,T_{\mathrm{NS}}-1$,
\[
    H_t=Z_t^\dagger Z_t,
    \qquad
    Z_{t+1}
    =
    aZ_t+Z_t(bH_t+cH_t^2),
\]
with fixed coefficients $a=3.4445$, $b=-4.7750$, and $c=2.0315$.
The output
$
    \RetrNS(Z,T_{\mathrm{NS}})=Z_{T_{\mathrm{NS}}}
$
serves as a decomposition-free approximate retraction toward the Stiefel
manifold. 

At iteration $k$, NS-RIS uses this map in two distinct ways; rigorous justification is provided in Proposition \ref{prop:operator-steepest-retraction}. First,
it is applied to the momentum buffer to obtain a polar-like, normalized search
direction,
\[
    \widetilde{M}_{k+1}=\RetrNS(M_{k+1},T_{\mathrm{NS}}),
\]
which plays the role of an orthogonalized descent direction. This step is
analogous to replacing the raw momentum by its polar factor, but avoids an
explicit singular value decomposition. Second, after moving along this
direction, the intermediate matrix generally no longer satisfies the Stiefel
constraint. Indeed, even if $\Gamma_k^\dagger\Gamma_k=\mathbb{I}_n$ and
$\widetilde{M}_{k+1}$ has approximately orthonormal columns, the matrix
$\Gamma_{k+\frac{1}{2}}=\Gamma_k-\eta_k\widetilde{M}_{k+1}$ need not satisfy
$\Gamma_{k+\frac{1}{2}}^\dagger\Gamma_{k+\frac{1}{2}}=\mathbb{I}_n$.
Therefore, NS-RIS applies the same Newton--Schulz orthogonalization a second
time as an efficient approximation to the polar projection of the updated
parameter matrix back to the feasible Stiefel manifold:
\[
    \Gamma_{k+\frac{1}{2}}=\Gamma_k-\eta_k\widetilde{M}_{k+1},
    \qquad
    \Gamma_{k+1}=\RetrNS(\Gamma_{k+\frac{1}{2}},T_{\mathrm{NS}}).
\]
Thus, the first call to $\RetrNS$ controls the geometry of the update
direction, whereas the second call enforces the constraint
$\Gamma_{k+1}^\dagger\Gamma_{k+1}=\mathbb{I}_n$ required by the stacked Kraus
representation of a valid HQMM. 

The updated matrix $\Gamma_{k+1}$ is finally
reshaped back into Kraus-operator form. At the end of each epoch, the stepsize
is decayed geometrically,
$
    \eta_{e+1}=\alpha\eta_e$, for 
    $0<\alpha\leq 1$,
where $\alpha$ is a fixed decay factor controlling how quickly the learning
rate decreases across epochs. When validation data are available, the
parameter tensor with the
largest validation log-likelihood is retained:
\[
    K_{\mathrm{best}}
    =
    \arg\max_{K_e}\ell_{\mathrm{val}}(K_e).
\]
The complete NS-RIS training procedure is summarized in
Algorithm~\ref{alg:learn-qgm4}. The Newton--Schulz orthogonalization subroutine used in
Algorithm~\ref{alg:learn-qgm4} is given explicitly in
Algorithm~\ref{alg:ns-orth}.

\begin{algorithm}
\small
\caption{NS-RIS for HQMM Learning}
\label{alg:learn-qgm4}
\begin{algorithmic}[1]
\Require Training sequences $D_{\mathrm{tr}}$, initial Kraus tensor $K_0$, initial stepsize $\eta_0$, momentum parameter $\beta$, decay factor $\alpha$, Newton--Schulz iteration count $T_{\mathrm{NS}}$, number of epochs $E$, optional validation set $D_{\mathrm{val}}$
\Ensure Best Kraus tensor $K_{\mathrm{best}}$
\State Convert $K_0$ to a Stiefel matrix $\Gamma \in \St(p,n)$, where $p=swn$
\State Set $K_{\mathrm{best}} \gets K_0$, $\ell_{\mathrm{best}} \gets -\infty$
\State Set stepsize $\eta \gets \eta_0$
\State Initialize momentum buffer $M \gets 0_{p \times n}$
\State Record initial train and validation metrics in $\mathcal{H}$
\For{$e = 1,\ldots,E$}
    \State Randomly permute the training sequence indices
    \State Reset epoch running totals
    \For{each sequence index $i$ in the permutation}
        \State Let $B \gets D_{\mathrm{tr}}[i,:]$
        \State Compute the gradient tensor $\nabla_K L(K;B)$ using the HQMM gradient routine
        \State Convert $\nabla_K L(K;B)$ to a matrix $G \in \mathbb{C}^{p \times n}$
        \State Project $G$ to the tangent space at $\Gamma$:
        \[
            G_R
            =
            G - \frac{1}{2}\Gamma\left(G^\dagger \Gamma + \Gamma^\dagger G\right)
        \]
        \State Update the momentum buffer:
        \[
            M \gets \beta M + (1-\beta)G_R
        \]
        \State Approximately orthogonalize the buffered direction:
        \[
            \widetilde{M} \gets \RetrNS(M,T_{\mathrm{NS}})
        \]
        \State Take a step and approximately project back to the Stiefel manifold:
        \[
            \Gamma_+ \gets \Gamma - \eta \widetilde{M},
            \qquad
            \Gamma \gets \RetrNS(\Gamma_+,T_{\mathrm{NS}})
        \]
        \State Convert $\Gamma$ back to the Kraus tensor $K$
    \EndFor
    \If{$D_{\mathrm{val}}$ is provided}
        \State Evaluate validation log-likelihood $\ell_{\mathrm{val}}$ 
        \If{$\ell_{\mathrm{val}} > \ell_{\mathrm{best}}$}
            \State $\ell_{\mathrm{best}} \gets \ell_{\mathrm{val}}$
            \State $K_{\mathrm{best}} \gets K$
        \EndIf
    \Else
        \State $K_{\mathrm{best}} \gets K$
    \EndIf
    \State Decay the stepsize: $\eta \gets \alpha\eta$
\EndFor
\State \Return $K_{\mathrm{best}}$
\end{algorithmic}
\end{algorithm}

\begin{algorithm}
\small
\caption{Newton--Schulz Approximate Orthogonalization}
\label{alg:ns-orth}
\begin{algorithmic}[1]
\Require Matrix $Z \in \mathbb{C}^{p \times n}$, number of iterations $T_{\mathrm{NS}}$
\Ensure Approximately orthogonalized matrix $Z$
\State Set coefficients $a=3.4445$, $b=-4.7750$, $c=2.0315$
\State Normalize $Z \gets Z / (\|Z\|_F + \epsilon)$
\If{$p < n$}
    \State $Z \gets Z^\dagger$
    \State Mark that $Z$ was transposed
\EndIf
\For{$t=1,\ldots,T_{\mathrm{NS}}$}
    \State $H \gets Z^\dagger Z$
    \State $P \gets bH + cH^2$
    \State $Z \gets aZ + ZP$
\EndFor
\If{$Z$ was transposed}
    \State $Z \gets Z^\dagger$
\EndIf
\State \Return $Z$
\end{algorithmic}
\end{algorithm}

\subsection{Mathematical Guarantees}
\label{subsec:nsris-guarantees}
We now state the assumptions and main mathematical guarantees for NS-RIS.
As $B$ is drawn i.i.d. from the empirical distribution over the
training set, we could, with a slight abuse of notation, write
\[
F(\Gamma)=\mathbb{E}_{B}[ L(\Gamma;B)],
\qquad
\Gamma\in\St(p,n),\quad p=swn.
\]
 We use the real Frobenius inner product
$\langle \Gamma, \Gamma'\rangle_F=\operatorname{Re}\operatorname{tr}(\Gamma^\dagger \Gamma')$.
The Riemannian gradient of $F$ at $\Gamma$ is
\[
    \operatorname{grad}F(\Gamma)
    =
    \Pi_{\Gamma}\bigl(\nabla F(\Gamma)\bigr),
\]
where $\nabla F(\Gamma)$ denotes the Euclidean gradient and $ \Pi_{\Gamma}(\cdot)$
is the orthogonal projection onto $T_\Gamma\St(p,n)$ (given in equation \eqref{eqn:projection}) under the real
Frobenius inner product. 

For any matrix $Z$, let $\Polar(Z)$ denote its thin polar factor. That is, if
$Z = U \Sigma V^\dagger$ is a thin singular value decomposition, then
\[
\Polar(Z) = UV^\dagger
\quad
\text{and}\quad
\langle Z, \Polar(Z) \rangle_F = \|Z\|_{1},
\]
where $\|\cdot\|_{1}$ denotes the Schatten--$1$  norm (also known as the nuclear norm). Throughout,
we measure smoothness with respect to the Schatten--$\infty$ norm,
$\|\cdot\|_{\infty}$  (also known as the operator norm), whose dual norm is the
Schatten--$1$ norm. This operator--nuclear geometry is particularly natural
for optimization algorithms based on polar factors or orthogonalized search
directions \citep{nesterov2013introductory,jaggi2013revisiting,beck2017first}, since
\[
\|Z\|_{1}
= \sup_{\|Y\|_{\infty}\le 1}\langle Z,Y\rangle_F
= \langle Z,\Polar(Z)\rangle_F,
\]
where the supremum is attained at $Y=\Polar(Z)$.

\begin{assumption}[Lipschitz smoothness]
	\label{assum:retraction-smoothness}
	The objective $F$ is bounded below by $F^*$ on $\St(p,n)$. Moreover, $F$ is continuously differentiable and $L$-Lipschitz smooth, i.e., for all  $\Gamma,\Gamma'\in\St(p,n)$
	there exists $L>0$ such that
	\begin{align*}
		\|\nabla F(\Gamma)-\nabla F(\Gamma')\|_1 \;\le\; L\,\|\Gamma-\Gamma'\|_\infty	\end{align*}
\end{assumption}

Assumption~\ref{assum:retraction-smoothness} is the standard smoothness
condition used in nonconvex first-order optimization. We therefore measure stationarity using the 
Schatten-$1$ norm of the Riemannian gradient. This criterion is stronger
than the usual Frobenius-norm criterion: since
$\|A\|_F\leq\|A\|_1$ for every matrix $A$, any point that is
$\epsilon$-stationary in nuclear norm is also $\epsilon$-stationary in
Frobenius norm. In this paper, the convergence metric is defined as follows:
\begin{definition}[$\epsilon$-stationary point]
	\label{def:stationarity}
	We call $\Gamma\in\mathbb{R}^{p \times n}$ an $\epsilon$-stationary point (in the nuclear norm) 
	if $\mathbb{E}[\|\operatorname{grad}F(\Gamma)\|_1]\le \epsilon$.
	Equivalently, we say an algorithm attains $\epsilon$-stationarity in $T$ steps if
	\begin{align*}
		\frac{1}{T}\sum_{t=1}^{T}\mathbb{E}[\|\operatorname{grad}F(\Gamma_{t-1})\|_1] \leq \epsilon.
	\end{align*}
\end{definition}

Proposition~\ref{prop:operator-steepest-retraction} explains the geometric role
of the Newton--Schulz steps in NS-RIS: the first step approximates the
operator-norm steepest descent direction through the polar factor of the
momentum, while the second step acts as an efficient Stiefel retraction that
restores feasibility after the update. The proof is deferred to
Appendix~\ref{sec:proofs}.
\begin{proposition}
\label{prop:operator-steepest-retraction}
Suppose Assumption~\ref{assum:retraction-smoothness} holds. Let
$M_{k+1}\in\mathbb{C}^{p\times n}$ be the momentum-smoothed 
gradient used by NS-RIS. For a feasible
first-order perturbation at $\Gamma_k$, represented by a curve $\gamma$ on
$\St(p,n)$ with initial velocity
$\Delta\in T_{\Gamma_k}\St(p,n)$, Lemma~\ref{lem:stiefel-first-order-expansion}
gives
\[
    F(\gamma(\eta_k))
    =
    F(\Gamma_k)
    +
    \eta_k
    \left\langle \operatorname{grad}F(\Gamma_k),\Delta\right\rangle_F
    +
    O(\eta_k^2).
\]
Replacing the full Riemannian gradient by its momentum-smoothed stochastic
surrogate $M_{k+1}$, NS-RIS uses the corresponding operator-norm linearized
model. The following ambient subproblem gives the steepest decrease direction
for this model; feasibility of the iterate is restored by the retraction step
below. Consider
\[
    \min_{\Delta\in\mathbb{C}^{p\times n}}
    \left\langle M_{k+1},\Delta\right\rangle_F
    \quad
    \text{subject to}
    \quad
    \|\Delta\|_\infty\leq 1 .
\]
If $M_{k+1}=U\Sigma V^\dagger$ is a thin singular value decomposition, then
an optimizer is
\[
    \Delta^\star=-UV^\dagger=-\Polar(M_{k+1}),
\]
and the optimal value is $-\|M_{k+1}\|_1$. Consequently, the first
Newton--Schulz orthogonalization in Algorithm~\ref{alg:learn-qgm4},
\[
    \widetilde{M}_{k+1}=\RetrNS(M_{k+1},T_{\mathrm{NS}}),
\]
computes a decomposition-free approximation of the operator-norm steepest
descent direction. The intermediate update
$\Gamma_{k+\frac12}=\Gamma_k-\eta_k\widetilde{M}_{k+1}$ is then mapped back
to the Stiefel manifold by the second Newton--Schulz orthogonalization. In
the exact polar projection case, this feasibility step solves
\[
    \min_{Q\in\St(p,n)}
    \|\Gamma_{k+\frac12}-Q\|_F^2
\]
and has the form
\[
    R(\Gamma_{k+\frac12})
    =
    \Gamma_{k+\frac12}
    \left(\Gamma_{k+\frac12}^{\dagger}
    \Gamma_{k+\frac12}\right)^{-1/2},
\]
and satisfies
$R(\Gamma_{k+\frac12})^\dagger R(\Gamma_{k+\frac12})=\mathbb{I}_n$ whenever
$\Gamma_{k+\frac12}$ has full column rank. The finite Newton--Schulz map
$\RetrNS(\Gamma_{k+\frac12},T_{\mathrm{NS}})$ is used as an efficient
approximation of this projection/retraction.
\end{proposition}

\begin{remark}[Frobenius-norm steepest direction]
\label{rem:frobenius-steepest}
The usual stochastic-gradient direction is the steepest descent direction
under the Frobenius norm. Indeed, for a nonzero projected stochastic gradient
$g_k$, consider
\[
    \min_{\Delta\in\mathbb{C}^{p\times n}}
    \left\langle g_k,\Delta\right\rangle_F
    \quad
    \text{subject to}
    \quad
    \|\Delta\|_F\leq 1 .
\]
By Cauchy--Schwarz,
$\langle g_k,\Delta\rangle_F\geq-\|g_k\|_F\|\Delta\|_F\geq-\|g_k\|_F$,
and equality is attained by
$\Delta^\star=-g_k/\|g_k\|_F$. Thus a Frobenius-norm steepest step uses the
normalized negative gradient direction. In contrast,
Proposition~\ref{prop:operator-steepest-retraction} replaces this
Frobenius-aligned direction by an operator-norm steepest direction obtained
from the polar structure of the matrix momentum.
\end{remark}

\begin{assumption}[Stochastic projected gradients]
\label{assum:stochastic-gradients}
At iteration $k$, the sample index $i_k$ is drawn uniformly from the
training set conditional on the history $\mathcal{F}_k$, and
\[
    g_k
    =
    \Pi_{\Gamma_k}\bigl(\nabla_{\Gamma}L(\Gamma_k;B_{i_k})\bigr)
\]
is an unbiased estimator of the Riemannian gradient:
$\mathbb{E}[g_k\mid\mathcal{F}_k]=\operatorname{grad}F(\Gamma_k)$.
There is a constant $\sigma^2<\infty$ such that
\[
    \mathbb{E}\!\left[
    \|g_k-\operatorname{grad}F(\Gamma_k)\|_F^2
    \mid\mathcal{F}_k
    \right]
    \leq \sigma^2 .
\]
For a mini-batch of size $b$, the right-hand side is replaced by
$\sigma^2/b$.
\end{assumption}

Assumption~\ref{assum:stochastic-gradients} is the standard unbiased
bounded-variance condition for stochastic gradient methods. It follows when
the mini-batch samples are drawn independently from the empirical
distribution and the per-sample projected gradients have bounded second
moment; on the compact Stiefel manifold this boundedness is automatic when
the per-sample gradients are continuous. The factor $\sigma^2/b$ records the
usual variance reduction from averaging $b$ i.i.d. samples. Assumption~\ref{assum:finite-ns} below only requires the finite Newton--Schulz
orthogonalization to approximate the exact polar direction uniformly within a
fixed tolerance; it does not require exact projection or an exact singular
value decomposition. This is mild in practice because Newton--Schulz
iteration for approximating the polar factor is a classical and well-studied
approach, with convergence properties established in early work on iterative
orthogonalization and polar decomposition and further refined by modern
analyses of Newton- and Halley-type iterations for matrix polar decomposition
and related matrix sign iterations~\citep{Bjorck1971,Higham1986,Nakatsukasa2010,Higham2008}.

\begin{assumption}[Finite NS accuracy]
\label{assum:finite-ns}
Define the exact polar direction $P_{k+1}=\Polar(M_{k+1})$ and the finite
Newton--Schulz direction
$\widetilde{M}_{k+1}=\RetrNS(M_{k+1},T_{\mathrm{NS}})$ used in
Algorithm~\ref{alg:learn-qgm4}.
For the chosen value of $T_{\mathrm{NS}}$, there is constant
$\varepsilon_{\mathrm{NS}}\in[0,1)$ such that
\[
    \|\widetilde{M}_{k+1}-P_{k+1}\|_{\mathrm{op}}\leq\varepsilon_{\mathrm{NS}},
\]
for all iterations considered.
\end{assumption}

Theorem~\ref{thm:nsris-convergence} gives the finite-time convergence
guarantee for NS-RIS. The proof is deferred to Appendix~\ref{sec:proofs}.
\begin{theorem}
	\label{thm:nsris-convergence}
	Suppose Assumptions~\ref{assum:retraction-smoothness},
	\ref{assum:stochastic-gradients}, and~\ref{assum:finite-ns} hold. Let
	$D=F(\Gamma_0)-F^*$, let
	\[
	    G_1\geq\sup_{\Gamma\in\St(p,n)}\|\nabla F(\Gamma)\|_1,
	    \qquad
	    L_R=2(L+G_1),
	\]
	and run
	Algorithm~\ref{alg:learn-qgm4} for $K$ inner iterations with mini-batch size
	$b$ and a fixed number $T_{\mathrm{NS}}$ of Newton--Schulz iterations.
	Let $0\leq\beta<1$, let $\eta_k=\eta>0$ be constant, and define
	\[
	    E_k=\Gamma_k-\bigl(\Gamma_{k-1}-\eta\widetilde{M}_{k}\bigr)
	    \quad\text{and}\quad
	    \rho_k=\|E_k\|_\infty .
	\]
	For
	\[
	    \bar\rho_K=\frac{1}{K}\sum_{k=1}^{K}\mathbb{E}\rho_k,
	    \qquad
	    \bar q_K=\frac{1}{K}\sum_{k=1}^{K}\mathbb{E}\rho_k^2,
	\]
	and
	\[
	\begin{aligned}
	    \mathcal{T}_K
	    =
	    \frac{1}{(1-\beta)K}
	    \mathbb{E}\|M_1-\operatorname{grad}F(\Gamma_0)\|_1
	    +
	    \frac{\sqrt{n}\,\sigma}{\sqrt{b}}
	    +
	    \frac{\beta L_R}{1-\beta}
	    \left(\eta (1+\varepsilon_{\mathrm{NS}})+\bar\rho_K\right).
	\end{aligned}
	\]
	the iterates satisfy
	\[
	\begin{aligned}
	\frac{1}{K}\sum_{k=0}^{K-1}
	\mathbb{E}\|\operatorname{grad}F(\Gamma_k)\|_1
	\leq\;&
	\frac{D}{\eta(1-\varepsilon_{\mathrm{NS}})K}
	+
	\frac{L_R(1+\varepsilon_{\mathrm{NS}})^2\eta}{1-\varepsilon_{\mathrm{NS}}}
	+
	\frac{2G_1\bar\rho_K+L_R\bar q_K}
	{\eta(1-\varepsilon_{\mathrm{NS}})}
+
\frac{2\mathcal{T}_K}{1-\varepsilon_{\mathrm{NS}}}
		 .
	\end{aligned}
	\]
\end{theorem}

\begin{remark}[Interpretation of the NS-RIS bound]
	\label{rem:nsris-convergence-interpretation}
	Theorem~\ref{thm:nsris-convergence} shows that NS-RIS drives the averaged
	Riemannian stationarity measure to an explicit neighborhood whose size is
	controlled by the telescoping descent term $D/(\eta K)$, the smoothness term
	$L_R\eta$, the finite Newton--Schulz feasibility residuals, and the tracking
	quantity $\mathcal{T}_K$. The tracking quantity $\mathcal{T}_K$ contains the
	initial momentum mismatch, the stochastic mini-batch noise
	$\sqrt{n}\sigma/\sqrt{b}$, and the drift of the Riemannian gradient along the
	finite-accuracy NS-RIS trajectory. For the
	single-sequence incremental version described in Algorithm~\ref{alg:learn-qgm4},
	$b=1$, so the stochastic neighborhood is governed by the single-sample
	variance level. The theorem
	says that the algorithm is stable under finite Newton--Schulz accuracy.
	
	The feasibility-correction contribution in the theorem is
	\[
	\frac{2G_1\bar\rho_K+L_R\bar q_K}
	{\eta(1-\varepsilon_{\mathrm{NS}})}
	+
	\frac{2 }{1-\varepsilon_{\mathrm{NS}}}
	\frac{\beta L_R}{1-\beta}\bar\rho_K .
	\]
	If the second Newton--Schulz retraction error satisfies
	$\bar\rho_K=O(\eta^2)$ and $\bar q_K=O(\eta^4)$, then the first fraction is
	\[
	O\!\left(\frac{\eta^2+\eta^4}{\eta}\right)
	=
	O(\eta),
	\]
	and the momentum-tracking residual term is $O(\eta^2)$ for fixed
	$\beta<1$. Therefore, the additional feasibility correction contributes only
	$O(\eta)$ to the overall stationarity bound.
	Consequently, with sufficiently accurate second Newton--Schulz
	orthogonalization, NS-RIS behaves like a stochastic Riemannian descent method
	up to controlled finite-accuracy and mini-batch noise floors. 
\end{remark}

\section{Numerical Analysis}
\label{sec:numerical-analysis}
This section describes the experimental setup and evaluation metric in
Subsection~\ref{sec:numerical-setup}, then reports performance on the
synthetic HMM benchmark in Subsection~\ref{sec:synthetic-hmm-results} and
the synthetic HQMM benchmark in Subsection~\ref{sec:synthetic-hqmm-results}.

\subsection{Setup and Evaluation Metric}
\label{sec:numerical-setup}

We use the experimental setting and evaluation convention as in
\citet{adhikary2020expressiveness}, which is a standard benchmark
protocol for HQMM learning algorithms. For an HQMM with latent dimension
$n$, output alphabet size $s$, and $w$ Kraus operators per output, we write
the model as an $(n,s,w)$-HQMM. Equivalently, after stacking all Kraus
operators vertically, the optimization variable is a complex Stiefel matrix
$\Gamma \in \mathbb{C}^{nsw \times n}$ whose columns satisfy
$\Gamma^\dagger \Gamma=\mathbb{I}_n$.

Unless otherwise stated, the latent density matrix is initialized as a
random Hermitian positive semidefinite matrix with unit trace, and the
stacked Kraus matrix is initialized as a random orthonormal matrix. Training
minimizes the negative log-likelihood of the observed sequences under the
HQMM. During optimization, validation performance is monitored after each
epoch, and the parameter tensor with the largest validation log-likelihood is
retained for final evaluation. This selection rule is equivalent to choosing
the model that assigns the highest probability to the validation set, while
leaving the reported test metric independent of the particular sequence
lengths used during training.

We compare against three standard learning procedures. The Givens Search
(GS) method of \citet{srinivasan2018} maintains feasibility by applying
local Givens rotations, or unitary transformations, to the stacked Kraus
matrix and accepting likelihood-improving updates. The constrained
optimization on the Stiefel manifold (COSM) method of
\citet{adhikary2020expressiveness} instead performs gradient-based updates
directly on the Stiefel manifold using a retraction that we propose. In the experiments of
\citet{adhikary2020expressiveness}, COSM converged to better optima faster
than GS and scaled to larger HQMMs that were too slow to train with GS. The
Expectation-Maximization (EM) algorithm is the standard maximum-likelihood
procedure for HMMs, alternating between inference of latent-state
responsibilities and parameter re-estimation; we use EM-trained HMMs as the
classical baseline. On the synthetic HMM benchmark,
\citet{adhikary2020expressiveness} report that small HQMMs can outperform
small HMMs, although this advantage does not hold for the $6$-state HMM
baseline, so EM remains a strong reference method when the data are generated
by an HMM.

In this section, we evaluate the synthetic HMM and HQMM
benchmarks. For the HMM benchmark, we follow the synthetic data setting of
\citet{srinivasan2018} and \citet{adhikary2020expressiveness}: the data are generated by an HMM with $6$ hidden
states and $6$ possible outputs, and we use the same $20$ training and $10$
validation sequences of length $3000$ as they did. Each long sequence is split into
$300$ shorter sequences, with a burn-in of $100$, before training HQMMs for
$60$ epochs and selecting the model with the highest validation description
accuracy. For the HQMM benchmark, we use the synthetic HQMM data in those two papers, generated by a
$2$-hidden-state, $6$-output HQMM. This benchmark uses the same $20$ training
and $10$ validation sequences of length $3000$, the same split into $300$
shorter sequences, and the same burn-in of $100$ for training; evaluation is
performed on $10$ test sequences of length $3000$ using a burn-in of $1000$.
This shortened training protocol reduces run time without changing the
amount of training data processed.

For all these experiments, we report description accuracy,
a scaled log-likelihood metric used in prior HQMM work
\citep{srinivasan2018,adhikary2020expressiveness}. Given a test sequence $Y$ of length
$\ell$, an alphabet of size $s$, and a trained model $\mathbb{D}$, the
description accuracy is
\[
    DA
    =
    f\left(
        1 + \frac{\log_s P(Y \mid \mathbb{D})}{\ell}
    \right),
\]
where
\[
    f(x)
    =
    \begin{cases}
        \tanh(x/8), & x \leq 0,\\
        x, & x > 0.
    \end{cases}
\]
The normalization by $\ell$ makes the score comparable across sequences of
different lengths. A value of $DA=1$ corresponds to assigning probability
one to the observed sequence, while $DA>0$ indicates performance better than
the uniform random baseline over the output alphabet. Thus, higher values of $DA$ indicate better predictive performance. When multiple test
sequences are evaluated, we report the mean description accuracy and use the
standard deviation across test sequences as the error bar. For labeled
sequence-classification experiments, we follow the same likelihood-based
decision rule as in HQMM benchmarks: one model is trained for each class, and
the predicted label is the class whose model assigns the highest likelihood
to the test sequence; the reported metric is average classification
accuracy.

\subsection{Synthetic HMM Test Performance}
\label{sec:synthetic-hmm-results}

Figure~\ref{fig:synth-hmm-w1} compares HQMM learning methods under different hidden dimensions and Kraus ranks. In the left panel, we fix the Kraus rank to $w=1$ and vary the hidden dimension $n$.  Under this setting, NS-RIS consistently achieves the highest test description accuracy among all HQMM training methods. The advantage becomes more pronounced as the hidden dimension increases, with NS-RIS substantially outperforming both GS and COSM for larger $n$, while also exceeding the corresponding classical EM baseline. In contrast, GS and COSM fail to surpass EM when $n=6$. 
In the right panel, we fix the hidden dimension at $n=6$ and vary the Kraus rank $w$ to evaluate the benefit of increasing the number of Kraus operators. Since EM does not use Kraus operators, its performance depends only on $(n,s)$ and is therefore independent of $w$.  NS-RIS maintains a clear performance advantage across all tested Kraus ranks, demonstrating that the proposed update rule can effectively exploit the richer HQMM parameterization. GS results are omitted for $w>1$, because the method becomes computationally prohibitive in these cases as revealed in  \citet{adhikary2020expressiveness}.

\begin{figure}[t!]
    \centering
    \includegraphics[width=1\linewidth]{\detokenize{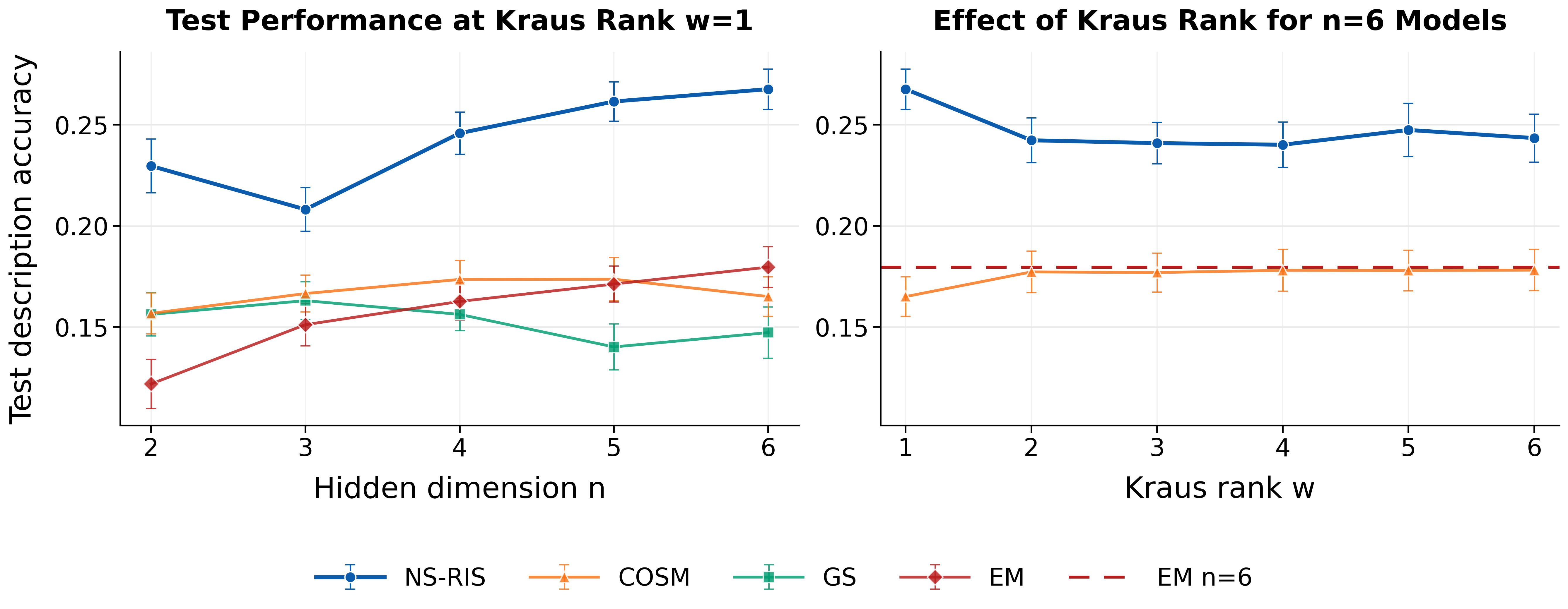}}
\caption{
	Test description accuracy for HQMM learning methods under varying hidden dimensions and Kraus ranks. 
	Left: Performance with Kraus rank fixed at $w=1$ while varying the hidden dimension $n$. We report the mean test description accuracy, with error bars indicating the standard deviation across test sequences. 
	Right: Performance for models with hidden dimension $n=6$ while varying the Kraus rank $w$. The EM baseline is shown as a horizontal dashed line because its performance is independent of $w$.
}
    \label{fig:synth-hmm-w1}
\end{figure}

In Figure~\ref{fig:synth-hmm-runtime}, we compare the running times of the HQMM learning methods under varying hidden dimensions and Kraus ranks. In the left panel, we fix the Kraus rank at $w=1$ and vary the hidden dimension $n$. NS-RIS and COSM exhibit nearly identical running times across all tested hidden dimensions, remaining relatively stable as $n$ increases. In contrast, GS incurs substantially higher computational cost, with runtime increasing steadily as the hidden dimension grows. This highlights the significantly better scalability of NS-RIS and COSM with respect to the latent dimension.
In the right panel, we fix the hidden dimension at $n=6$ and vary the Kraus rank $w$. Both NS-RIS and COSM experience increased computational cost as the Kraus rank grows, reflecting the larger HQMM parameterization. However, NS-RIS remains competitive with COSM across all tested Kraus ranks and is slightly faster for larger values of $w$. Combined with the accuracy results in Figure~\ref{fig:synth-hmm-w1}, these findings demonstrate that NS-RIS achieves stronger predictive performance without incurring additional computational overhead relative to existing scalable HQMM training methods.

\begin{figure}[t!]
    \centering
    \includegraphics[width=1\linewidth]{\detokenize{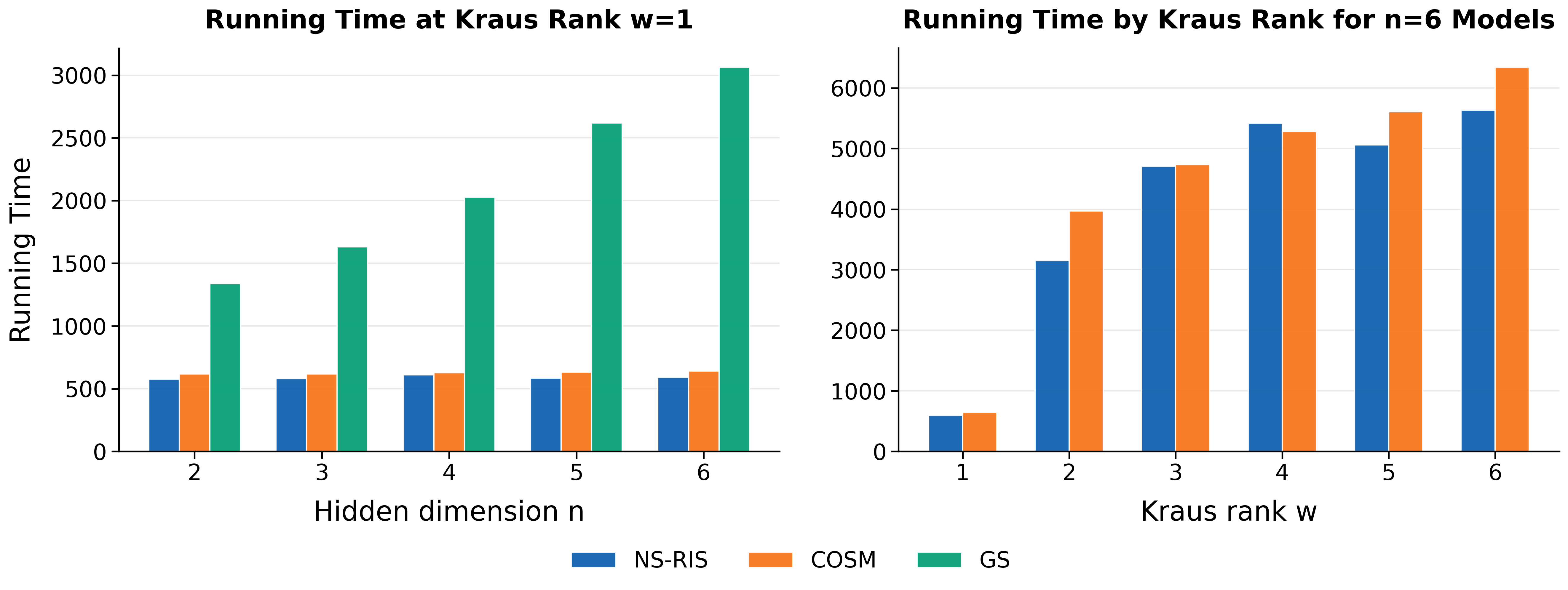}}
    \caption{
    	Running time comparison for HQMM learning methods under varying hidden dimensions and Kraus ranks.
    	Left: Running time with Kraus rank fixed at $w=1$ while varying the hidden dimension $n$. NS-RIS and COSM maintain relatively stable runtimes across hidden dimensions, whereas GS becomes substantially more expensive as $n$ increases.
    	Right: Running time for models with hidden dimension $n=6$ while varying the Kraus rank $w$. Both NS-RIS and COSM require additional computation for larger Kraus ranks, but NS-RIS remains competitive with COSM across all tested settings.
    }
    \label{fig:synth-hmm-runtime}
\end{figure}

Across the evaluated synthetic benchmark settings, NS-RIS consistently outperforms COSM in terms of test performance. To quantify the relative gain, we compute the percentage improvement using $$\mathrm{Improvement}(\%) = \frac{\mathrm{Metric}*{\text{NS-RIS}} - \mathrm{Metric}*{\text{COSM}}}{\mathrm{Metric}_{\text{COSM}}} \times 100.$$ 
Averaged across all experimental configurations, NS-RIS achieves approximately $38.5\%$ relative improvement over COSM. 
The largest improvement is observed for $\texttt{scenario\_id}=4$, where the averaged test metric increases from $0.1736$ for COSM to $0.2615$ for NS-RIS. 
Substituting these values into the relative-improvement formula gives $\frac{0.2615 - 0.1736}{0.1736} \times 100 = 50.6\%$, demonstrating that NS-RIS attains over $50\%$ relative improvement in the best-performing configuration. These results indicate that the proposed NS-RIS optimization framework provides substantial and consistent performance gains over existing geometric optimization approaches for synthetic HMM learning tasks.

It is worth noting that only approximately $10\%$ of the experimental settings involved explicit hyperparameter tuning. Specifically, hyperparameter tuning was conducted only for the hidden-state configuration with $n=4$, while for all remaining experiments we directly adopted the best-performing hyperparameter settings reported for COSM in~\cite{adhikary2020expressiveness}. Despite this minimal tuning effort, NS-RIS consistently achieved superior predictive performance and competitive runtime efficiency across the evaluated benchmarks. This result highlights the robustness and strong generalization capability of the proposed NS-RIS optimization framework, suggesting that it is substantially less sensitive to hyperparameter selection while still converging to high-quality solutions. The same phenomenon was observed in the subsequent HQMM benchmark experiments.

\subsection{Synthetic HQMM Test Performance}
\label{sec:synthetic-hqmm-results}
The synthetic HQMM benchmark follows the experimental setting introduced by \cite{adhikary2020expressiveness}, where the data are generated from a synthetic HQMM inspired by the Stern--Gerlach experiment in quantum mechanics. The benchmark uses a configuration with hidden dimension $n=2$, output alphabet size $s=6$, and Kraus rank $w=1$, which was specifically designed to highlight the expressive advantages of HQMMs over classical HMMs. In the original study, the authors demonstrated that a significantly larger classical HMM is required to match the representational capacity of a small HQMM, emphasizing the richer expressiveness of quantum-inspired latent-state models.  Figure~\ref{fig:synth-hqmm-runtime} presents the averaged training, validation, and test metrics together with standard-deviation error bars across repeated runs, as well as the running-time comparison between NS-RIS, GS, COSM, and EM. The results show that NS-RIS consistently achieves the strongest predictive performance across all evaluation metrics while maintaining competitive computational efficiency. In particular, NS-RIS improves the test metric over COSM by approximately $18.9\%$, computed as $\frac{0.1536 - 0.1292}{0.1292}\times100$, while simultaneously reducing runtime by approximately $12.0\%$ relative to COSM, computed as $\frac{643.23 - 566.19}{643.23}\times100$. These results demonstrate that the proposed NS-RIS can improve both solution quality and practical scalability for HQMM learning. It is worth mentioning that no additional hyperparameter tuning was conducted in our experiments. All results were generated using the same hyperparameter configuration reported as optimal for the COSM method in~\cite{adhikary2020expressiveness}, and this identical configuration was applied uniformly across NS-RIS, GS, COSM, and EM to ensure a fair and controlled comparison between optimization methods.

\begin{figure}[t!]
	\centering
	\includegraphics[width=1\linewidth]{\detokenize{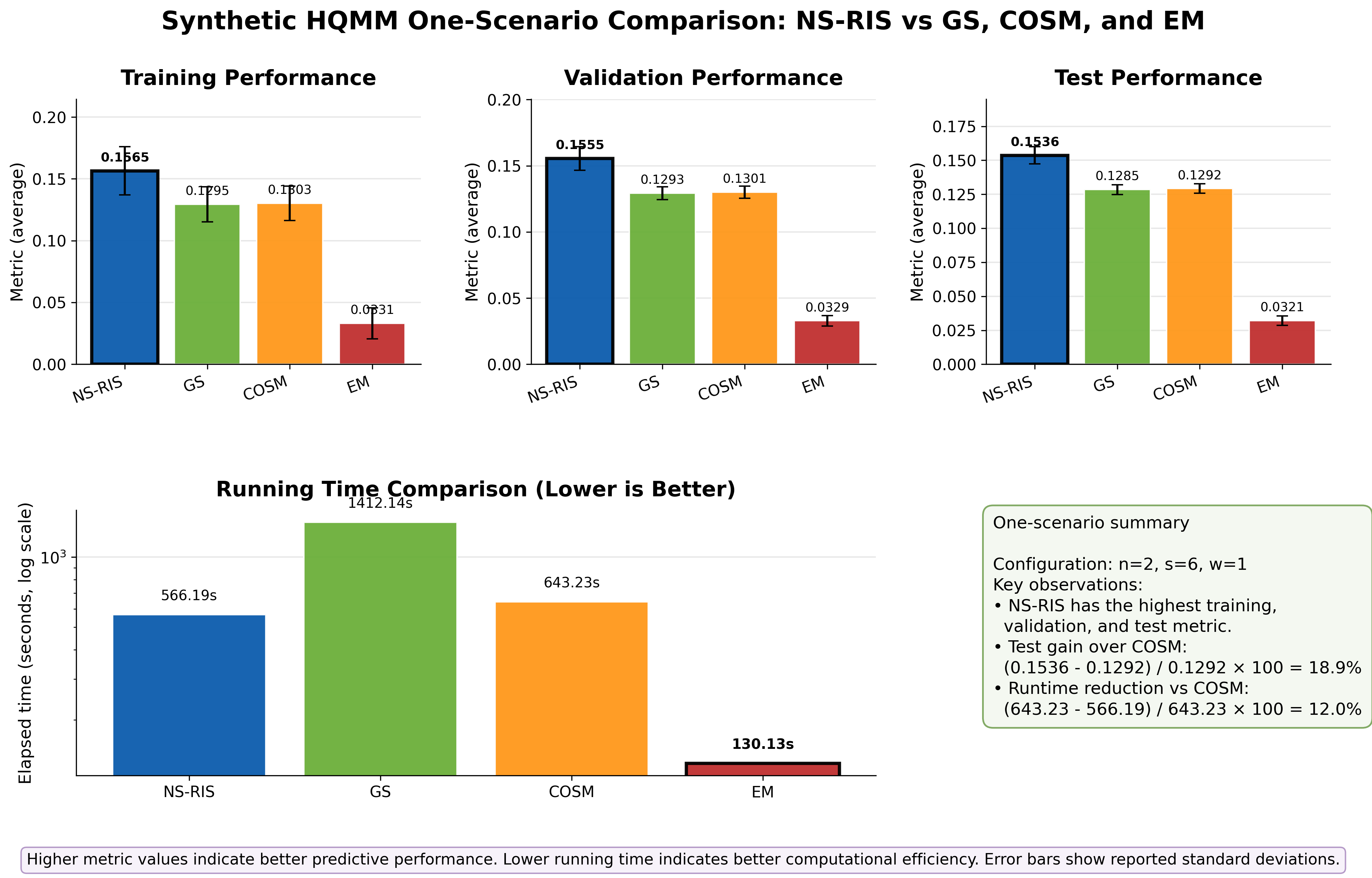}}
\caption{Comparison of NS-RIS, GS, COSM, and EM on the synthetic HQMM benchmark with configuration $(n=2, s=6, w=1)$. The top row reports the averaged training, validation, and test metrics with standard-deviation error bars across repeated runs, while the bottom row compares elapsed running times on a logarithmic scale. NS-RIS consistently achieves the best predictive performance across all evaluation metrics while also providing competitive computational efficiency.}
	\label{fig:synth-hqmm-runtime}
\end{figure}

\section{Empirical Analysis}
\label{sec:empirical-analysis}
This section evaluates NS-RIS on the real Splice benchmark. We first describe
the dataset and preprocessing in Subsection~\ref{sec:splice-dataset}, then
report the classification and runtime results in
Subsection~\ref{sec:splice-results}, and finally discuss the biological
explanation for using higher-dimensional latent spaces in
Subsection~\ref{sec:scientific-validation}.

\subsection{The Splice Dataset}
\label{sec:splice-dataset}
The splice dataset is a widely used benchmark for sequence classification and splice-junction prediction in computational biology \citep{Towell1991,Dheeru2017} that is publicly available at UCI machine learning repository. The dataset consists of DNA sequences of fixed length $60$, where each position corresponds to one of the four nucleobases: Adenine (A), Cytosine (C), Guanine (G), or Thymine (T). Biologically, DNA sequences contain protein-coding regions known as exons, interspersed with non-coding regions called introns. Correct identification of exon--intron boundaries is an important problem in gene prediction and genomic analysis. The classification task is to determine whether a sequence contains an exon--intron (EI) splice site, an intron--exon (IE) splice site, or neither (N). The dataset contains $762$ EI examples, $765$ IE examples, and $1648$ negative examples. Although the original dataset includes several ambiguous nucleotide symbols beyond ${A,C,G,T}$, we remove sequences containing ambiguous characters during preprocessing in order to ensure a consistent discrete input alphabet for training and evaluation.

\subsection{Splice Classification Results}
\label{sec:splice-results}
Figure~\ref{fig:splice_combined_mean_error_bars} compares mean classification
error for COSM and NS-RIS over latent dimensions $n\in\{2,4,6,8\}$ and Kraus
ranks $w\in\{1,2,4,6\}$. The results show that increasing the latent
dimension is beneficial for NS-RIS: for $n=6$ and $n=8$, NS-RIS consistently
achieves lower error than the corresponding COSM configurations and falls
well below the EM baseline, with the best mean error reaching approximately
$0.340$ for $n=6$ and $0.330$ for $n=8$. More specifically, at $n=6$,
the best COSM error is $0.422$, whereas NS-RIS reaches $0.340$, a relative
reduction of about $19.4\%$; this is also substantially below the EM baseline
error $0.428$. At $n=8$, the best COSM error is $0.402$, while NS-RIS
achieves $0.330$, a relative reduction of about $17.9\%$ and an improvement
over the EM baseline error $0.421$. For smaller latent dimensions ($n=2$
and $n=4$), the performance is more mixed, indicating that the advantage of
NS-RIS becomes most visible when the HQMM has enough latent capacity to model
the splice-junction structure.
\begin{figure}[t!]
	\centering
	\includegraphics[width=1\linewidth]{\detokenize{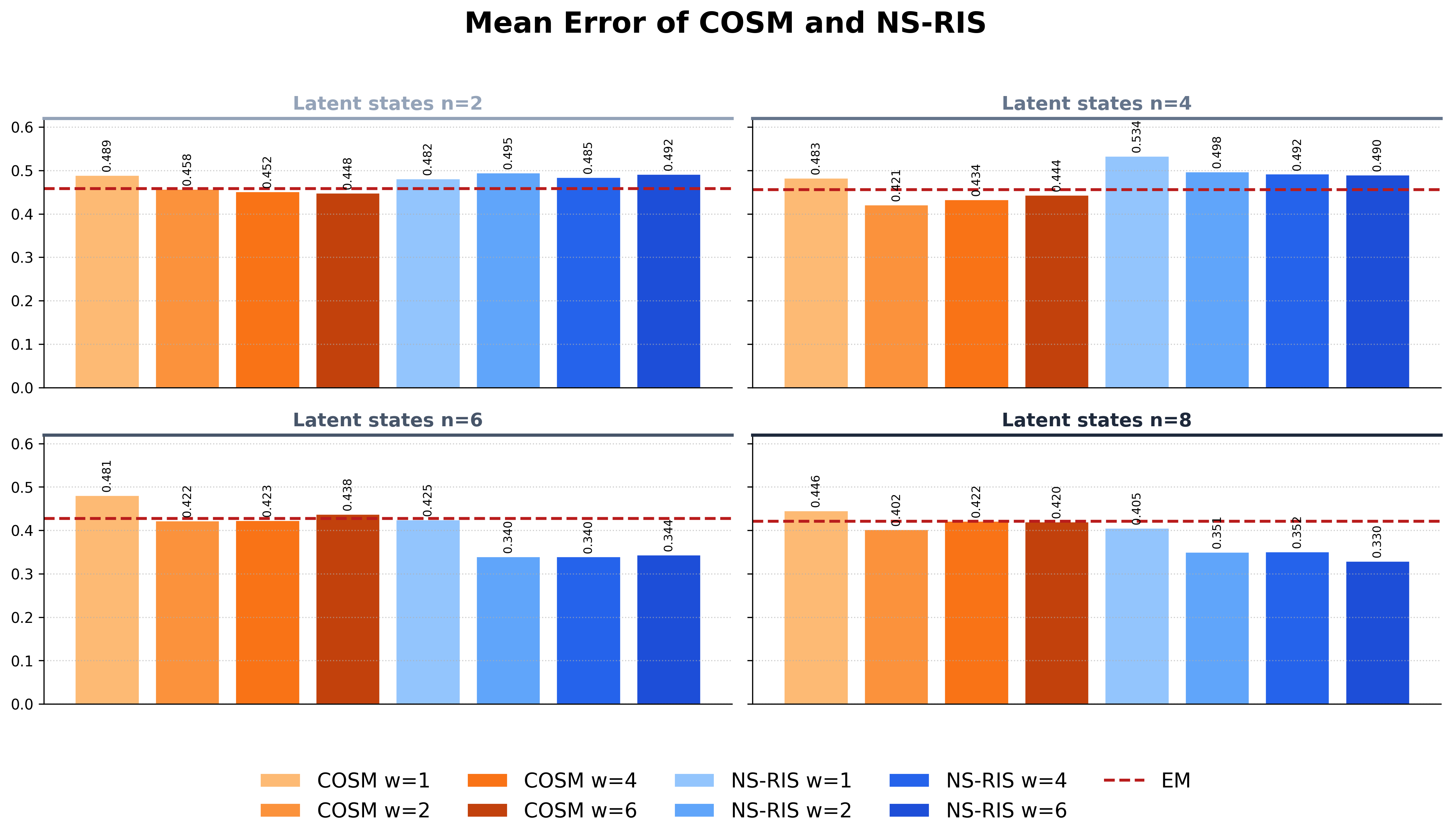}}
	\caption{Mean error comparison of COSM and NS-RIS on the Splice dataset for latent state sizes $n \in \{2,4,6,8\}$. Within each panel, bars report the mean error for window sizes $w \in \{1,2,4,6\}$, with darker shades indicating larger $w$. The dashed horizontal line shows the EM baseline for the same latent state size: $0.458$ for $n=2$, $0.456$ for $n=4$, $0.428$ for $n=6$, and $0.421$ for $n=8$.}
	\label{fig:splice_combined_mean_error_bars}
\end{figure}

Figure~\ref{fig:three_metrics} gives a class-wise comparison between COSM
and NS-RIS using the EI, IE, and negative-class error metrics. The nested
pies show that NS-RIS improves all three class-specific errors in the
larger-latent-state settings $(n,w)=(6,4)$ and $(8,4)$, with particularly
large reductions in the negative-class error. For $(n,w)=(6,4)$, NS-RIS
reduces the EI error from $0.422160$ to $0.349046$ ($17.32\%$ lower), the IE
error from $0.357189$ to $0.328432$ ($8.05\%$ lower), and the negative-class
error from $0.544604$ to $0.409422$ ($24.82\%$ lower). For $(n,w)=(8,4)$,
NS-RIS reduces the EI error from $0.389734$ to $0.340166$ ($12.72\%$ lower),
the IE error from $0.346732$ to $0.344771$ ($0.57\%$ lower), and the
negative-class error from $0.530787$ to $0.393054$ ($25.95\%$ lower). In the
smaller settings, NS-RIS is not uniformly better across all classes: for
$(2,4)$ it improves IE error from $0.388889$ to $0.311438$ but increases EI
and negative-class errors, while for $(4,4)$ all three errors are higher than
COSM. This class-wise view reinforces that the strongest advantage appears in
the larger latent-state regimes, where NS-RIS improves not only the average
error but also all three biologically meaningful classification categories.
\begin{figure}[t!]
	\centering
	\includegraphics[width=0.98\linewidth]{\detokenize{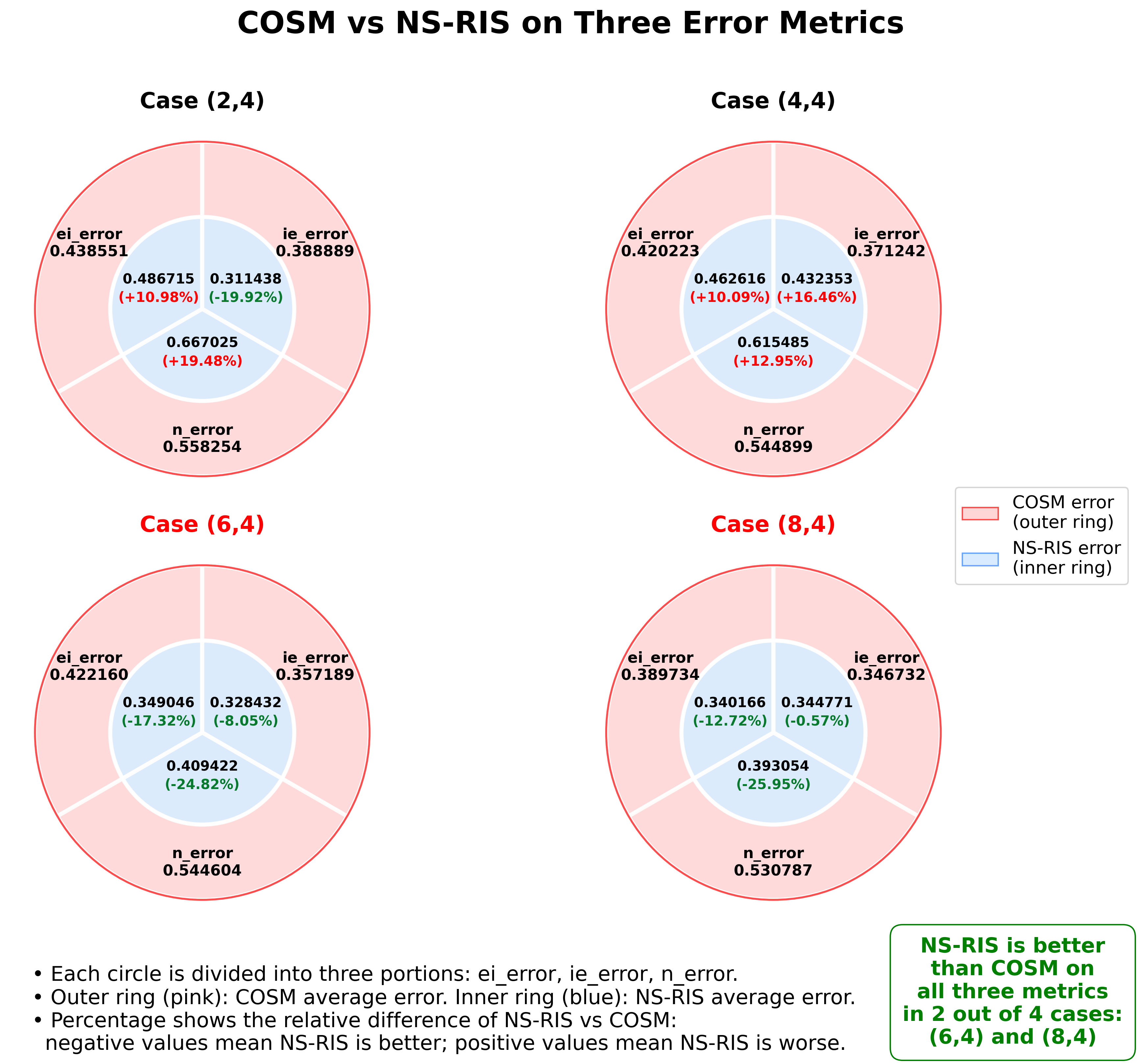}}
	\caption{Nested pie comparison of COSM and NS-RIS errors across three metrics.}
	\label{fig:three_metrics}
\end{figure}

Figure~\ref{fig:splice_runtime_nsris_cosm_bars} compares the running time of
NS-RIS and COSM across the tested Splice configurations. The upper subfigure
shows that COSM generally becomes slower as the Kraus rank $w$ increases,
although its ordering is not perfectly monotone across all latent dimensions.
The lower subfigure shows a clearer monotone pattern for NS-RIS: the running
time grows first with the Kraus rank $w$ and then with the combined scale
$n w$, which more fully reflects the dimension of the stacked Stiefel
parameter. The runtime advantage of NS-RIS is especially clear for several
larger Kraus-rank settings where COSM requires substantially more time. Thus the
improved classification accuracy is not obtained by paying a larger
computational cost; rather, the Newton--Schulz retraction strategy improves
predictive performance while preserving practical scalability.
\begin{figure}[t!]
	\centering
	\includegraphics[width=0.95\linewidth]{\detokenize{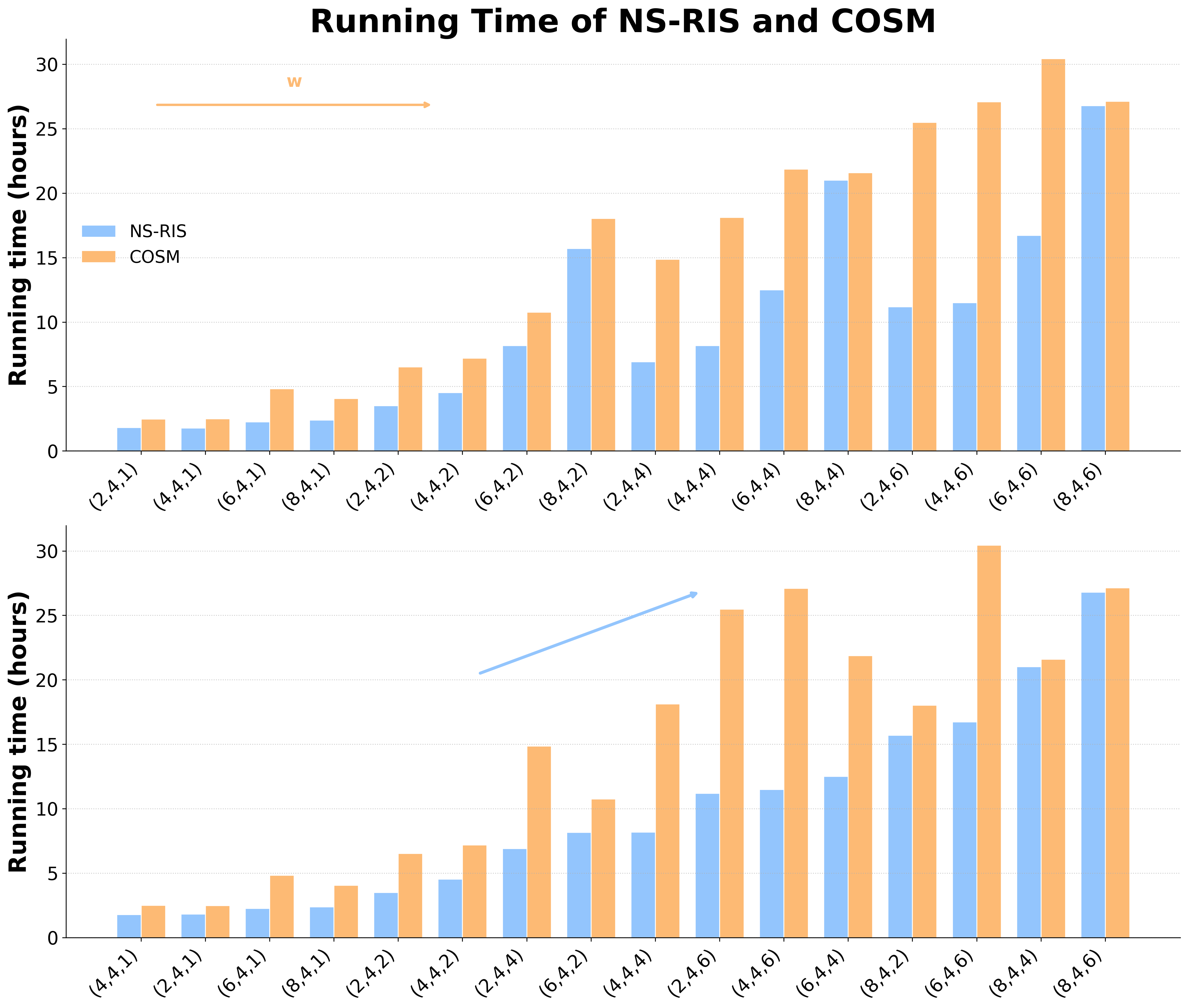}}
	\caption{Running-time comparison of NS-RIS and COSM on the Splice dataset across latent dimensions and Kraus ranks.}
	\label{fig:splice_runtime_nsris_cosm_bars}
\end{figure}

\subsection{Scientific Validation}
\label{sec:scientific-validation}
Recent advances in quantum computing and quantum machine learning have stimulated growing interest in applying quantum-inspired methods to genomic and DNA sequence analysis. Several studies have explored quantum algorithms for biological sequence comparison, DNA sequence alignment, genome assembly, and genomic pattern recognition \citep{Outeiral2020,KosogluKind2023,Varsamis2023}. Broader surveys of quantum computing in biology likewise identify genetics, molecular biology, drug design, and bioinformatics as application areas where quantum representations may become useful \citep{ghamsari2025quantum}. In particular, quantum-based approaches have been proposed to model high-dimensional sequence dependencies and complex biological correlations that are difficult to capture using classical probabilistic models alone \citep{nalkecz2024quantum}. Recent work has also demonstrated the feasibility of encoding and analyzing genomic sequences on real quantum hardware, highlighting the potential of quantum representations for large-scale genomic inference tasks. These developments suggest that richer latent-state frameworks, such as HQMMs, may provide advantages over classical HMMs in modeling splice-junction sequences, especially when long-range dependencies and higher-order contextual interactions are present in biological sequence data.

Although the observed alphabet in the splice dataset consists of only four nucleotides ${A,C,G,T}$, there is strong evidence from computational biology and sequence modeling that the latent state space should generally be substantially larger than four. In HMMs and related state-space approaches for genomic sequences, the latent variables are not intended to represent the nucleotides themselves, but rather higher-order biological contexts such as splice donor/acceptor motifs, exon and intron regions, insertion/deletion dynamics, codon structure, and long-range dependencies. Classical profile-HMM formulations for biological sequence analysis therefore employ many more hidden states than observable symbols, including match, insertion, and deletion states at different sequence positions \citep{Krogh1994,Eddy1998}. In splice-site prediction specifically, different latent states are used to model biologically distinct regimes surrounding exon--intron boundaries, even when the emissions remain limited to four nucleobases \citep{Burge1998}. Consequently, restricting the latent dimensionality to four would force the model to conflate multiple heterogeneous biological mechanisms into the same representation, reducing its ability to capture contextual sequence structure.

More recent representation-learning approaches for DNA sequence modeling further support the use of latent spaces whose dimensionality exceeds the cardinality of the nucleotide alphabet. Deep latent-variable and sequence-embedding models trained on splice-site datasets consistently learn distributed representations that encode motif composition, positional dependencies, and regulatory patterns that cannot be represented using only four discrete latent categories \citep{Agarwal2019,DeepDeCode2023}. 
From an information-theoretic perspective, the observable alphabet size only constrains the emission space, whereas the latent state dimension reflects the complexity of the underlying generative process. Since splice recognition depends on combinatorial sequence motifs and contextual interactions extending across many nucleotide positions, using a latent dimension larger than four is both biologically motivated and empirically supported in prior literature.

\section{Conclusion}
\label{sec:conclusion}
This paper introduces NS-RIS for scalable learning HQMM parameters on the complex Stiefel manifold. NS-RIS enables HQMM training to be both geometrically faithful and computationally efficient. Specifically, the first Newton--Schulz iteration approximates the operator-norm steepest descent direction via the polar factor of the momentum, while the second Newton--Schulz iteration restores the Stiefel-manifold feasibility required by the trace-preserving Kraus constraint.

We establish the theoretical convergence properties of NS-RIS under standard assumptions on smoothness, stochastic gradients, and finite Newton--Schulz accuracy. In particular, we derive an explicit finite-time stationarity bound that separates the contributions of the telescoping descent term, smoothness error, stochastic momentum-tracking error, and feasibility residuals, thereby clarifying how each component influences convergence. The analysis further shows that sufficiently accurate Newton--Schulz retraction errors contribute only lower-order terms, implying that the decomposition-free implementation retains the convergence behavior of stochastic Riemannian gradient descent up to the inherent stochastic noise floor.

Extensive experiments on both synthetic and real-world sequence datasets demonstrate the effectiveness of NS-RIS. On synthetic HMM and HQMM benchmarks, NS-RIS consistently achieves stronger predictive performance than existing HQMM training methods while remaining computationally competitive. On the Splice benchmark, NS-RIS is particularly effective when the latent dimension is sufficiently large to capture the biological complexity of splice-junction sequences, reducing both the overall classification error and class-specific errors compared with the current state-of-the-art method in higher-dimensional settings.

Importantly, NS-RIS provides the
first evidence in these benchmarks that an HQMM can significantly outperform the
HMM baseline on data that are not generated by a quantum model. 
This is a remarkable step beyond the theoretical fact that HQMMs generalize
HMMs: it shows that, when equipped with scalable Stiefel-manifold inference,
HQMMs can deliver practical gains on ordinary sequence data. These results
therefore position HQMMs as viable replacements or extensions of HMMs across
the broad scientific sequence-modeling settings where HMMs have long been
standard tools and richer latent dynamics are needed.

\appendix
\section{Proofs}
\label{sec:proofs}

This appendix gives the proofs of the mathematical statements in
Subsection~\ref{subsec:nsris-guarantees}. Subsection~\ref{app:proof-proposition}
proves Proposition~\ref{prop:operator-steepest-retraction}, and
Subsection~\ref{app:proof-theorem} proves
Theorem~\ref{thm:nsris-convergence} together with the auxiliary lemmas used
in the convergence analysis.

\subsection{Proof of Proposition~\ref{prop:operator-steepest-retraction}}
\label{app:proof-proposition}
Lemma \ref{lem:stiefel-first-order-expansion} below gives the first-order expansion on the Stiefel manifold.

\begin{lemma}
\label{lem:stiefel-first-order-expansion}
Let $F$ be continuously differentiable in a neighborhood of
$\St(p,n)$, and let $\Gamma\in\St(p,n)$. If
$\gamma:(-\epsilon,\epsilon)\to\St(p,n)$ is a differentiable curve satisfying
$\gamma(0)=\Gamma$ and $\gamma'(0)=\xi\in T_\Gamma\St(p,n)$, then
\[
    F(\gamma(t))
    =
    F(\Gamma)
    +
    t\left\langle \operatorname{grad}F(\Gamma),\xi\right\rangle_F
    +
    o(t).
\]
If the gradient of $F$ is locally Lipschitz, the remainder is
$O(t^2)$; in particular, this quadratic remainder holds under
Assumption~\ref{assum:retraction-smoothness}. In particular, for a straight perturbation
$\Gamma+t\Delta$, the same first-order expression with
$\operatorname{grad}F(\Gamma)$ is valid when
$\Delta\in T_\Gamma\St(p,n)$. For a general ambient perturbation
$\Delta\in\mathbb{C}^{p\times n}$, the linear term is instead
$\langle\nabla F(\Gamma),\Delta\rangle_F$.
\end{lemma}

\begin{proof}
By the chain rule and the definition of the Euclidean gradient,
\[
    \frac{d}{dt}F(\gamma(t))\bigg|_{t=0}
    =
    \left\langle \nabla F(\Gamma),\xi\right\rangle_F .
\]
Since $\operatorname{grad}F(\Gamma)=\Pi_\Gamma(\nabla F(\Gamma))$ is the
orthogonal projection of $\nabla F(\Gamma)$ onto
$T_\Gamma\St(p,n)$, the residual
$\nabla F(\Gamma)-\operatorname{grad}F(\Gamma)$ is orthogonal to every
tangent vector. Because $\xi\in T_\Gamma\St(p,n)$,
\[
    \left\langle \nabla F(\Gamma),\xi\right\rangle_F
    =
    \left\langle \operatorname{grad}F(\Gamma),\xi\right\rangle_F .
\]
This proves the stated first-order expansion. The first-order
$o(t)$ term only uses differentiability. If the gradient is locally
Lipschitz, the usual Taylor estimate gives a quadratic remainder; this
condition is supplied by Assumption~\ref{assum:retraction-smoothness} in the
operator--nuclear geometry used here. The final claim follows by applying
the same calculation to the straight curve
$t\mapsto\Gamma+t\Delta$; when $\Delta$ is not tangent, the normal component
of $\Delta$ need not be orthogonal to $\nabla F(\Gamma)$, so the Euclidean
gradient gives the correct ambient first variation.
\end{proof}

\begin{proof}[Proof of Proposition~\ref{prop:operator-steepest-retraction}]
By spectral--nuclear duality, for every feasible $\Delta$,
\[
    \left\langle M_{k+1},\Delta\right\rangle_F
    \geq
    -\|M_{k+1}\|_1\|\Delta\|_\infty
    \geq
    -\|M_{k+1}\|_1 .
\]
Taking $\Delta^\star=-UV^\dagger$ gives
$\|\Delta^\star\|_\infty=1$ and
\[
    \left\langle M_{k+1},\Delta^\star\right\rangle_F
    =
    -\left\langle U\Sigma V^\dagger,UV^\dagger\right\rangle_F
    =
    -\operatorname{tr}(\Sigma)
    =
    -\|M_{k+1}\|_1 .
\]
Thus $-\Polar(M_{k+1})$ solves the constrained linearized decrease problem.
For the projection statement, let
$Y=\Gamma_{k+\frac12}$ and
$R(Y)=Y(Y^\dagger Y)^{-1/2}$. For $Q\in\St(p,n)$,
\[
    \|Y-Q\|_F^2
    =
    \operatorname{tr}(Y^\dagger Y)
    +
    n
    -
    2\operatorname{Re}\operatorname{tr}(Q^\dagger Y).
\]
The first two terms do not depend on $Q$, so the projection problem is
equivalent to maximizing
$\operatorname{Re}\operatorname{tr}(Q^\dagger Y)$. To see why the polar
factor attains this maximum, let $Y=A\Sigma B^\dagger$ be a thin singular
value decomposition. Then
$Y(Y^\dagger Y)^{-1/2}=AB^\dagger$. For any $Q\in\St(p,n)$,
\[
    \operatorname{Re}\operatorname{tr}(Q^\dagger Y)
    =
    \operatorname{Re}\operatorname{tr}(B^\dagger Q^\dagger A\Sigma)
    \leq
    \sum_{j=1}^n \sigma_j(Y),
\]
because $B^\dagger Q^\dagger A$ is a contraction and the trace is maximized
when its diagonal entries are all equal to one. This upper bound is achieved
by choosing $Q=AB^\dagger=R(Y)$, since then
$\operatorname{Re}\operatorname{tr}(Q^\dagger Y)=\operatorname{tr}(\Sigma)$.
Thus the polar factor $R(Y)=Y(Y^\dagger Y)^{-1/2}$ is the closest Stiefel
matrix to $Y$ in Frobenius norm. Moreover,
\[
    R(Y)^\dagger R(Y)
    =
    (Y^\dagger Y)^{-1/2}Y^\dagger Y(Y^\dagger Y)^{-1/2}
    =
    \mathbb{I}_n,
\]
which is the definition of Stiefel manifold.
\end{proof}

\subsection{Proof of Theorem~\ref{thm:nsris-convergence}}
\label{app:proof-theorem}

The following four lemmas, Lemmas~\ref{thm:operator-nuclear-descent}--\ref{lem:momentum-tracking},
will be used to prove Theorem~\ref{thm:nsris-convergence}.

\begin{lemma}
	\label{thm:operator-nuclear-descent}
	Suppose Assumption~\ref{assum:retraction-smoothness} holds. For $\Gamma\in\St(p,n)$, define
	\[
	S_\Gamma
	=
	\frac{\Gamma^\dagger\nabla F(\Gamma)
		+\nabla F(\Gamma)^\dagger\Gamma}{2}.
	\]
	Then the Riemannian gradient satisfies
	\[
	\operatorname{grad}F(\Gamma)
	=
	\nabla F(\Gamma)-\Gamma S_\Gamma .
	\]
	Moreover, for any $\Gamma,\Gamma'\in\St(p,n)$,
	\[
	\begin{aligned}
		F(\Gamma')
		\leq\;&
		F(\Gamma)
		+
		\left\langle
		\operatorname{grad}F(\Gamma),\Gamma'-\Gamma
		\right\rangle_F
		+
		\frac{L}{2}\|\Gamma'-\Gamma\|_\infty^2 -
		\frac{1}{2}
		\left\langle
		S_\Gamma,
		(\Gamma'-\Gamma)^\dagger(\Gamma'-\Gamma)
		\right\rangle_F .
	\end{aligned}
	\]
	Consequently,
	\[
	F(\Gamma')
	\leq
	F(\Gamma)
	+
	\left\langle
	\operatorname{grad}F(\Gamma),\Gamma'-\Gamma
	\right\rangle_F
	+
	\frac{L+\|S_\Gamma\|_1}{2}
	\|\Gamma'-\Gamma\|_\infty^2 .
	\]
	Letting
	$G_1\geq \sup_{\Gamma\in\St(p,n)}\|\nabla F(\Gamma)\|_1$, then
	\[
	F(\Gamma')
	\leq
	F(\Gamma)
	+
	\left\langle
	\operatorname{grad}F(\Gamma),\Gamma'-\Gamma
	\right\rangle_F
	+
	\frac{L+G_1}{2}
	\|\Gamma'-\Gamma\|_\infty^2 .
	\]
\end{lemma}

\begin{proof}
	Let $\Delta=\Gamma'-\Gamma$. Define
$\Gamma_t=\Gamma+t\Delta$, for $0\leq t\leq 1$. 
	By the fundamental theorem of calculus,
	\[
	F(\Gamma')-F(\Gamma)
	=
	\int_0^1
	\left\langle \nabla F(\Gamma_t),\Delta\right\rangle_F\,dt .
	\]
	Adding and subtracting $\nabla F(\Gamma)$ inside the integrand gives
	\[
	\begin{aligned}
	F(\Gamma')-F(\Gamma)
	=
	\left\langle\nabla F(\Gamma),\Delta\right\rangle_F +
	\int_0^1
	\left\langle
	\nabla F(\Gamma_t)-\nabla F(\Gamma),\Delta
	\right\rangle_F\,dt .
	\end{aligned}
	\]
	Using spectral--nuclear duality,
	\[
	\left\langle
	\nabla F(\Gamma_t)-\nabla F(\Gamma),\Delta
	\right\rangle_F
	\leq
	\|\nabla F(\Gamma_t)-\nabla F(\Gamma)\|_1
	\|\Delta\|_\infty .
	\]
	By Assumption~\ref{assum:retraction-smoothness}, applied along the segment
	joining $\Gamma$ and $\Gamma'$,
	\[
	\|\nabla F(\Gamma_t)-\nabla F(\Gamma)\|_1
	\leq
	L\|\Gamma_t-\Gamma\|_\infty
	=
	Lt\|\Delta\|_\infty .
	\]
	Therefore,
	\[
	\int_0^1
	\left\langle
	\nabla F(\Gamma_t)-\nabla F(\Gamma),\Delta
	\right\rangle_F\,dt
	\leq
	L\|\Delta\|_\infty^2\int_0^1 t\,dt
	=
	\frac{L}{2}\|\Delta\|_\infty^2 .
	\]
	Combining the preceding displays yields
	\[
	F(\Gamma')
	\leq
	F(\Gamma)
	+
	\langle  \nabla F(\Gamma),\Delta\rangle_F
	+
	\frac{L}{2}\|\Delta\|_\infty^2 .
	\]
The tangent space of the complex Stiefel manifold is
\[
T_\Gamma\St(p,n)
=
\{\xi\in\mathbb{C}^{p\times n}:
\Gamma^\dagger\xi+\xi^\dagger\Gamma=0\},
\]
which follows by differentiating the constraint
$\Gamma^\dagger\Gamma=\mathbb{I}_n$ along a differentiable curve on
$\St(p,n)$. More explicitly, if $\gamma(t)\in\St(p,n)$,
$\gamma(0)=\Gamma$, and $\gamma'(0)=\xi$, then differentiating
$\gamma(t)^\dagger\gamma(t)=\mathbb{I}_n$ at $t=0$ gives
$\xi^\dagger\Gamma+\Gamma^\dagger\xi=0$.

We next identify the tangent and normal parts of the Euclidean gradient
directly. The matrix $S_\Gamma$ is Hermitian by construction. Moreover, for
any Hermitian matrix $S$ and any $\xi\in T_\Gamma\St(p,n)$,
\[
\langle \Gamma S,\xi\rangle_F
=
\operatorname{Re}\operatorname{tr}((\Gamma S)^\dagger\xi)
=
\operatorname{Re}\operatorname{tr}(S\Gamma^\dagger\xi).
\]
The tangent-space relation implies that $\Gamma^\dagger\xi$ is
skew-Hermitian because
\[
(\Gamma^\dagger\xi)^\dagger
=
\xi^\dagger\Gamma
=
-\Gamma^\dagger\xi .
\]
For any Hermitian matrix $H=H^\dagger$ and skew-Hermitian matrix
$K^\dagger=-K$, the real Frobenius inner product is zero:
\[
\langle H,K\rangle_F
=
\operatorname{Re}\operatorname{tr}(H^\dagger K)
=
\operatorname{Re}\operatorname{tr}(HK).
\]
Since $(\operatorname{tr}(HK))^\ast
=\operatorname{tr}((HK)^\dagger)
=\operatorname{tr}(K^\dagger H)
=-\operatorname{tr}(KH)
=-\operatorname{tr}(HK)$, the scalar $\operatorname{tr}(HK)$ is purely
imaginary, and its real part is zero. Thus
$\langle\Gamma S,\xi\rangle_F=0$ for every tangent vector $\xi$. In
particular, $\Gamma S_\Gamma$ is normal to $T_\Gamma\St(p,n)$.

It remains to check that the remaining term is tangent. Let
$\zeta_\Gamma=\nabla F(\Gamma)-\Gamma S_\Gamma$. Using
$\Gamma^\dagger\Gamma=\mathbb{I}_n$ and $S_\Gamma=S_\Gamma^\dagger$, we have
\[
\Gamma^\dagger\zeta_\Gamma+\zeta_\Gamma^\dagger\Gamma
=
\Gamma^\dagger\nabla F(\Gamma)
+\nabla F(\Gamma)^\dagger\Gamma
-2S_\Gamma
=
0,
\]
where the last equality follows from the definition of $S_\Gamma$. Hence
$\zeta_\Gamma\in T_\Gamma\St(p,n)$. Therefore
\[
\nabla F(\Gamma)
=
\zeta_\Gamma+\Gamma S_\Gamma
\]
is an orthogonal decomposition into a tangent component and a normal
component. By the definition of the Riemannian gradient as the tangent
projection of the Euclidean gradient,
\[
\operatorname{grad}F(\Gamma)
=
\Pi_\Gamma(\nabla F(\Gamma))
=
\nabla F(\Gamma)-\Gamma S_\Gamma .
\]
Consequently,
\[
\nabla F(\Gamma)
=
\operatorname{grad}F(\Gamma)+\Gamma S_\Gamma .
\]
Substitution into the Euclidean descent inequality yields
\[
F(\Gamma')
\leq
F(\Gamma)
+
\langle\operatorname{grad}F(\Gamma),\Delta\rangle_F
	+
	\langle\Gamma S_\Gamma,\Delta\rangle_F
	+
	\frac{L}{2}\|\Delta\|_\infty^2 .
	\]

Now, we are going to express the normal-gradient contribution as a second-order
term. Since both endpoints are feasible and $\Gamma'=\Gamma+\Delta$,
\[
\mathbb{I}_n
=
\Gamma'^\dagger\Gamma'
=
(\Gamma+\Delta)^\dagger(\Gamma+\Delta)
=
\mathbb{I}_n
+
\Gamma^\dagger\Delta+\Delta^\dagger\Gamma
+
\Delta^\dagger\Delta,
\]
so
\[
\Gamma^\dagger\Delta+\Delta^\dagger\Gamma
=
-\Delta^\dagger\Delta .
\]
Decompose $\Gamma^\dagger\Delta$ into Hermitian and skew-Hermitian
parts:
\[
\Gamma^\dagger\Delta
=
\frac{\Gamma^\dagger\Delta+\Delta^\dagger\Gamma}{2}
+
\frac{\Gamma^\dagger\Delta-\Delta^\dagger\Gamma}{2}.
\]
The second term is skew-Hermitian. Since $S_\Gamma$ is Hermitian, it is
orthogonal to this skew-Hermitian part under
$\langle A,B\rangle_F=\operatorname{Re}\operatorname{tr}(A^\dagger B)$.
Therefore, only the Hermitian part of $\Gamma^\dagger\Delta$ contributes to
the inner product with $S_\Gamma$. Using the feasibility identity above,
\[
\frac{\Gamma^\dagger\Delta+\Delta^\dagger\Gamma}{2}
=
-\frac{1}{2}\Delta^\dagger\Delta .
\]
Also,
\[
\langle\Gamma S_\Gamma,\Delta\rangle_F
=
\langle S_\Gamma,\Gamma^\dagger\Delta\rangle_F
=
\left\langle
S_\Gamma,
\frac{\Gamma^\dagger\Delta+\Delta^\dagger\Gamma}{2}
\right\rangle_F
=
-\frac{1}{2}
\langle S_\Gamma,\Delta^\dagger\Delta\rangle_F .
\]
Substituting this identity into the descent inequality proves the first
bound. The correction is quadratic because it depends on
$\Delta^\dagger\Delta$, even though $\Delta=\Gamma'-\Gamma$ is generally not
a tangent vector at $\Gamma$.

	For the second bound, use spectral--nuclear duality:
	\[
	-\frac{1}{2}\langle S_\Gamma,\Delta^\dagger\Delta\rangle_F
	\leq
	\frac{1}{2}
	\left|
	\langle S_\Gamma,\Delta^\dagger\Delta\rangle_F
	\right|
	\leq
	\frac{1}{2}\|S_\Gamma\|_1\|\Delta^\dagger\Delta\|_\infty .
	\]
	Furthermore, by submultiplicativity,
	\[
	\|\Delta^\dagger\Delta\|_\infty
	\leq
	\|\Delta^\dagger\|_\infty\|\Delta\|_\infty
	=
	\|\Delta\|_\infty^2 .
	\]
	Combining these two estimates gives the second bound.
	
	Finally, we control $\|S_\Gamma\|_1$. Since
	$S_\Gamma=(\Gamma^\dagger\nabla F(\Gamma)+\nabla F(\Gamma)^\dagger\Gamma)/2$,
	the triangle inequality gives
	\[
	\|S_\Gamma\|_1
	\leq
	\frac{1}{2}
	\|\Gamma^\dagger\nabla F(\Gamma)\|_1
	+
	\frac{1}{2}
	\|\nabla F(\Gamma)^\dagger\Gamma\|_1 .
	\]
	The two terms are equal, and because $\Gamma^\dagger\Gamma=\mathbb{I}_n$,
	all singular values of $\Gamma$ are equal to one; hence
	$\|\Gamma^\dagger\|_\infty=1$. Therefore,
	\[
	\|S_\Gamma\|_1
	\leq
	\|\Gamma^\dagger\nabla F(\Gamma)\|_1
	\leq
	\|\nabla F(\Gamma)\|_1
	\leq
	G_1.
	\]
Substituting this estimate into the second bound gives the final statement.
\end{proof}
%
%\begin{remark}
%	\label{rem:stiefel-smooth-role}
%	Lemma \ref{thm:operator-nuclear-descent} provides the deterministic smoothness
%	estimate used later in the one-step descent argument for NS-RIS. It explains
%	why replacing the Euclidean gradient by its tangent projection introduces only
%	a quadratic correction when both endpoints lie on the Stiefel manifold. The
%	bounded-gradient constant $G$ is finite whenever the Euclidean gradient is
%	continuous on $\St(p,n)$, since $\St(p,n)$ is compact.
%\end{remark}

The following Lemma \ref{lem:riemannian-gradient-lipschitz} shows that the Euclidean Lipschitz smoothness assumption
implies a corresponding Lipschitz bound for the Riemannian gradient after
projection onto the Stiefel tangent space.
\begin{lemma}
\label{lem:riemannian-gradient-lipschitz}
Suppose Assumption~\ref{assum:retraction-smoothness} holds, and let
$G_1\geq\sup_{\Gamma\in\St(p,n)}\|\nabla F(\Gamma)\|_1$ as in Lemma \ref{thm:operator-nuclear-descent}. Then, for all
$\Gamma,\Gamma'\in\St(p,n)$,
\[
    \|\operatorname{grad}F(\Gamma)-\operatorname{grad}F(\Gamma')\|_1
    \leq
    L_R\|\Gamma-\Gamma'\|_\infty,
    \qquad
    L_R:=2(L+G_1).
\]
\end{lemma}

\begin{proof}
Let $A=\nabla F(\Gamma)$, $A'=\nabla F(\Gamma')$, and
$d=\|\Gamma-\Gamma'\|_\infty$. By
Assumption~\ref{assum:retraction-smoothness},
\[
    \|A-A'\|_1\leq Ld .
\]
Using
$\operatorname{grad}F(\Gamma)=\nabla F(\Gamma)-\Gamma S_\Gamma$, we have
\[
\begin{aligned}
    \|\operatorname{grad}F(\Gamma)-\operatorname{grad}F(\Gamma')\|_1
    &\leq
    \|A-A'\|_1
    +
    \|\Gamma S_\Gamma-\Gamma'S_{\Gamma'}\|_1  \\
    &\leq
    Ld
    +
    \|(\Gamma-\Gamma')S_\Gamma\|_1
    +
    \|\Gamma'(S_\Gamma-S_{\Gamma'})\|_1 .
\end{aligned}
\]
Since $\Gamma,\Gamma'\in\St(p,n)$, $\|\Gamma'\|_\infty=1$. Also
$\|S_\Gamma\|_1\leq\|\nabla F(\Gamma)\|_1\leq G_1$, as shown in
Lemma~\ref{thm:operator-nuclear-descent}. Therefore, by
submultiplicativity of the nuclear norm with respect to the operator norm,
\[
    \|(\Gamma-\Gamma')S_\Gamma\|_1
    \leq
    \|\Gamma-\Gamma'\|_\infty\|S_\Gamma\|_1
    \leq
    G_1d,
\]
and
\[
    \|\Gamma'(S_\Gamma-S_{\Gamma'})\|_1
    \leq
    \|\Gamma'\|_\infty\|S_\Gamma-S_{\Gamma'}\|_1
    =
    \|S_\Gamma-S_{\Gamma'}\|_1 .
\]
Combining these two estimates with $\|A-A'\|_1\leq Ld$ gives
\[
    \|\operatorname{grad}F(\Gamma)-\operatorname{grad}F(\Gamma')\|_1
    \leq
    (L+G_1)d+\|S_\Gamma-S_{\Gamma'}\|_1 .
\]
It remains to bound the last term. From the definition of $S_\Gamma$,
with
\[
    B=\Gamma^\dagger A-\Gamma'^\dagger A',
\]
we have
\[
    S_\Gamma-S_{\Gamma'}
    =
    \frac{1}{2}
    \left(B+B^\dagger\right),
\]
because
\[
    B^\dagger
    =
    A^\dagger\Gamma-A'^\dagger\Gamma' .
\]
Since the nuclear norm is invariant under adjoints, $\|B^\dagger\|_1=\|B\|_1$.
Thus
\[
    \|S_\Gamma-S_{\Gamma'}\|_1
    \leq
    \frac{1}{2}\|B\|_1+\frac{1}{2}\|B^\dagger\|_1
    =
    \|B\|_1 .
\]
That is,
\[
\begin{aligned}
    \|S_\Gamma-S_{\Gamma'}\|_1
    &\leq
    \|\Gamma^\dagger A-\Gamma'^\dagger A'\|_1  \\
    &\leq
    \|\Gamma^\dagger(A-A')\|_1
    +
    \|(\Gamma-\Gamma')^\dagger A'\|_1  \\
    &\leq
    Ld+G_1d .
\end{aligned}
\]
Combining the two estimates gives
\[
    \|\operatorname{grad}F(\Gamma)-\operatorname{grad}F(\Gamma')\|_1
    \leq
    2(L+G_1)d
    =
    L_R\|\Gamma-\Gamma'\|_\infty .
\]
\end{proof}

The following Lemma \ref{lem:nsris-descent} gives the one-step descent estimate for the actual NS-RIS
update, including both the finite Newton--Schulz direction error and the
second Newton--Schulz feasibility correction.
\begin{lemma}
	\label{lem:nsris-descent}
Consider the actual NS-RIS update
	\[
	    \Gamma_{k+\frac{1}{2}}
	    =
	    \Gamma_k-\eta\widetilde{M}_{k+1},
	    \qquad
	    \Gamma_{k+1}
	    =
	    \RetrNS(\Gamma_{k+\frac{1}{2}},T_{\mathrm{NS}}),
	\]
	and suppose that $\Gamma_{k+1}\in\St(p,n)$. Define the second
	Newton--Schulz correction
	\[
	    E_{k+1}
	    =
	    \Gamma_{k+1}-\Gamma_{k+\frac{1}{2}},
	    \qquad
	    \rho_{k+1}=\|E_{k+1}\|_\infty .
	\]
	Then, under Assumptions~\ref{assum:retraction-smoothness} and~\ref{assum:finite-ns},
	\[
	\begin{aligned}
		F(\Gamma_{k+1})
		\leq\;&
		F(\Gamma_k)
		-
		\eta(1-\varepsilon_{\mathrm{NS}})\|M_{k+1}\|_1
		\\
		&+
		\eta(1+\varepsilon_{\mathrm{NS}})
		\|\operatorname{grad}F(\Gamma_k)-M_{k+1}\|_1
		+
		\rho_{k+1}\|\operatorname{grad}F(\Gamma_k)\|_1
		\\
		&+
		\frac{L_R}{2}
		\left(
		    \eta\|\widetilde{M}_{k+1}\|_\infty
		    +
		    \rho_{k+1}
		\right)^2 .
	\end{aligned}
	\]
\end{lemma}

\begin{proof}
	Let
	\[
	    H_k=\operatorname{grad}F(\Gamma_k)
	    \quad\text{and}\quad
	    \Delta_k=\Gamma_{k+1}-\Gamma_k .
	\]
	The actual update gives the exact decomposition
	\[
	    \Delta_k
	    =
	    -\eta\widetilde{M}_{k+1}+E_{k+1}.
	\]
	By Lemma~\ref{thm:operator-nuclear-descent}, with the weaker quadratic
	constant $L_R\geq L+G_1$,
	\[
	F(\Gamma_{k+1})
	\leq
	F(\Gamma_k)
	-
	\eta\langle H_k,\widetilde{M}_{k+1}\rangle_F
	+
	\langle H_k,E_{k+1}\rangle_F
	+
	\frac{L_R}{2}\|\Delta_k\|_\infty^2 .
	\]
	Spectral--nuclear duality and the definition of $\rho_{k+1}$ give
	\[
	    \langle H_k,E_{k+1}\rangle_F
	    \leq
	    \|H_k\|_1\|E_{k+1}\|_\infty
	    =
	    \rho_{k+1}\|H_k\|_1,
	\]
	and the triangle inequality gives
	\[
	    \|\Delta_k\|_\infty
	    \leq
	    \eta\|\widetilde{M}_{k+1}\|_\infty+\rho_{k+1}.
	\]
	The inner product is decomposed into a momentum term and a tracking error:
	\[
	\begin{aligned}
		\langle H_k,\widetilde{M}_{k+1}\rangle_F
		&=
		\langle M_{k+1},\widetilde{M}_{k+1}\rangle_F
		+
		\langle H_k-M_{k+1},\widetilde{M}_{k+1}\rangle_F .
	\end{aligned}
	\]
	Let $P_{k+1}=\Polar(M_{k+1})$. Since
	$\langle M_{k+1},P_{k+1}\rangle_F=\|M_{k+1}\|_1$,
	\[
	\langle M_{k+1},\widetilde{M}_{k+1}\rangle_F
	=
	\|M_{k+1}\|_1
	+
	\langle M_{k+1},\widetilde{M}_{k+1}-P_{k+1}\rangle_F
	\geq
	(1-\varepsilon_{\mathrm{NS}})\|M_{k+1}\|_1,
	\]
	where the last inequality uses spectral--nuclear duality and
	Assumption~\ref{assum:finite-ns}. The same assumption also implies
	\[
	    \|\widetilde{M}_{k+1}\|_\infty
	    \leq
	    \|P_{k+1}\|_\infty
	    +
	    \|\widetilde{M}_{k+1}-P_{k+1}\|_\infty
	    \leq
	    1+\varepsilon_{\mathrm{NS}} .
	\]
	Thus the remaining term is bounded by
	\[
	\left|
	    \langle H_k-M_{k+1},\widetilde{M}_{k+1}\rangle_F
	\right|
	\leq
	(1+\varepsilon_{\mathrm{NS}})
	\|H_k-M_{k+1}\|_1 .
	\]
	Combining the inequalities proves the claim.
\end{proof}

The following Lemma \ref{lem:momentum-tracking} controls how closely the momentum variable tracks the
Riemannian gradient along the stochastic NS-RIS iterates.
\begin{lemma}
	\label{lem:momentum-tracking}
	For a constant stepsize $\eta$ and a fixed momentum
	parameter $0\leq\beta<1$, let
	\[
	    H_k=\operatorname{grad}F(\Gamma_k),
	    \qquad
	    e_{k+1}=M_{k+1}-H_k .
	\]
	For $k\geq1$, define the second Newton--Schulz correction
	\[
	    E_k
	    =
	    \Gamma_k-\bigl(\Gamma_{k-1}-\eta\widetilde{M}_{k}\bigr),
	    \qquad
	    \rho_k=\|E_k\|_\infty .
	\]
	Under  Assumptions~\ref{assum:retraction-smoothness}
	and~\ref{assum:stochastic-gradients}, the exponential moving average
	$M_{k+1}=\beta M_k+(1-\beta)g_k$ satisfies
	\[
	\begin{aligned}
	\frac{1}{K}\sum_{k=0}^{K-1}
	\mathbb{E}
	\|H_k-M_{k+1}\|_1
	\leq\;&
	\frac{\mathbb{E}\|e_1\|_1}{(1-\beta)K}
	+
	\frac{\sqrt{n}\,\sigma}{\sqrt{b}}
+
	\frac{\beta L_R}{1-\beta}
	\frac{1}{K}\sum_{k=1}^{K-1}
	\mathbb{E}\left[
	    \eta\|\widetilde{M}_{k}\|_\infty+\rho_k
	\right].
	\end{aligned}
	\]
\end{lemma}

\begin{proof}
	Write
	\[
	    g_k=H_k+\zeta_k,
	    \qquad \text{where}\quad
	    \mathbb{E}[\zeta_k\mid\mathcal{F}_k]=0 .
	\]
	By Assumption~\ref{assum:stochastic-gradients},
	\[
	    \mathbb{E}\!\left[
	    \|\zeta_k\|_F^2\mid\mathcal{F}_k
	    \right]
	    \leq
	    \frac{\sigma^2}{b}.
	\]
	Since $\operatorname{rank}(\zeta_k)\leq n$,
	\[
	    \mathbb{E}\|\zeta_k\|_1
	    \leq
	    \sqrt{n}\,\mathbb{E}\|\zeta_k\|_F
	    \leq
	    \frac{\sqrt{n}\,\sigma}{\sqrt{b}} .
	\]
	For $k\geq1$, the momentum recursion gives
	\[
	\begin{aligned}
	    e_{k+1}
	    &=
	    M_{k+1}-H_k  \\
	    &=
	    \beta M_k+(1-\beta)g_k-H_k  \\
	    &=
	    \beta(M_k-H_{k-1})
	    +
	    \beta(H_{k-1}-H_k)
	    +
	    (1-\beta)\zeta_k  \\
	    &=
	    \beta e_k
	    +
	    \beta(H_{k-1}-H_k)
	    +
	    (1-\beta)\zeta_k .
	\end{aligned}
	\]
	Hence,
	\[
	    \mathbb{E}\|e_{k+1}\|_1
	    \leq
	    \beta\mathbb{E}\|e_k\|_1
	    +
	    \beta\mathbb{E}\|H_k-H_{k-1}\|_1
	    +
	    (1-\beta)\frac{\sqrt{n}\,\sigma}{\sqrt{b}} .
	\]
	By Lemma~\ref{lem:riemannian-gradient-lipschitz} and the actual update,
	\[
	    \Gamma_k-\Gamma_{k-1}
	    =
	    -\eta\widetilde{M}_k+E_k,
	\]
we have
	\[
	    \|H_k-H_{k-1}\|_1
	    \leq
	    L_R\|\Gamma_k-\Gamma_{k-1}\|_\infty
	    \leq
	    L_R\left(\eta\|\widetilde{M}_k\|_\infty+\rho_k\right).
	\]
	Let $a_k=\mathbb{E}\|e_k\|_1$ and
	$s=\sqrt{n}\sigma/\sqrt{b}$. Summing the preceding recursion from
	$k=1$ to $K-1$ gives
	\[
	    \sum_{k=1}^{K-1}a_{k+1}
	    \leq
	    \beta\sum_{k=1}^{K-1}a_k
	    +
	    \beta L_R
	    \sum_{k=1}^{K-1}
	    \mathbb{E}\left[
	        \eta\|\widetilde{M}_k\|_\infty+\rho_k
	    \right]
	    +
	    (K-1)(1-\beta)s .
	\]
	Since
	$\sum_{k=1}^{K-1}a_{k+1}=\sum_{k=2}^{K}a_k$ and
	$\sum_{k=1}^{K-1}a_k=a_1+\sum_{k=2}^{K}a_k-a_K$, this implies
	\[
	    (1-\beta)\sum_{k=2}^{K}a_k
	    \leq
	    \beta a_1
	    +
	    \beta L_R
	    \sum_{k=1}^{K-1}
	    \mathbb{E}\left[
	        \eta\|\widetilde{M}_k\|_\infty+\rho_k
	    \right]
	    +
	    (K-1)(1-\beta)s,
	\]
	where the nonpositive term $-\beta a_K$ has been dropped. Adding $a_1$
	to both sides after division by $1-\beta$ gives
	\[
	    \sum_{k=1}^{K}a_k
	    \leq
	    \frac{a_1}{1-\beta}
	    +
	    (K-1)s
	    +
	    \frac{\beta L_R}{1-\beta}
	    \sum_{k=1}^{K-1}
	    \mathbb{E}\left[
	        \eta\|\widetilde{M}_k\|_\infty+\rho_k
	    \right].
	\]
After the final
	division by $K$, the noise contribution is
	$((K-1)/K)s\leq s$.
	Dividing by $K$ and recalling that
	$\|H_k-M_{k+1}\|_1=\|e_{k+1}\|_1$ proves the claim.
\end{proof}

\begin{proof}[Proof of Theorem~\ref{thm:nsris-convergence}]
	For compactness, write
	\[
	H_k=\operatorname{grad}F(\Gamma_k)\quad \text{and}\quad
	T_k=\|H_k-M_{k+1}\|_1.
	\]
	Lemma~\ref{lem:nsris-descent} gives the one-step inequality
	\[
	\begin{aligned}
		F(\Gamma_{k+1})
		\leq\;&
		F(\Gamma_k)
		-
		\eta (1-\varepsilon_{\mathrm{NS}})\|M_{k+1}\|_1
		+
		\eta (1+\varepsilon_{\mathrm{NS}})T_k
		+
		\rho_{k+1}\|H_k\|_1
		+
		\frac{L_R}{2}
		\left(
		    \eta\|\widetilde{M}_{k+1}\|_\infty+\rho_{k+1}
		\right)^2 .
	\end{aligned}
	\]
	By Assumption~\ref{assum:finite-ns},
	$\|\widetilde{M}_{k+1}\|_\infty\leq (1+\varepsilon_{\mathrm{NS}})$. Also,
	$\|H_k\|_1\leq\|\nabla F(\Gamma_k)\|_1+\|S_{\Gamma_k}\|_1\leq2G_1$.
	Using $(x+y)^2\leq2x^2+2y^2$ gives
	\[
	\eta (1-\varepsilon_{\mathrm{NS}})\|M_{k+1}\|_1
	\leq
	F(\Gamma_k)-F(\Gamma_{k+1})
	+
	\eta (1+\varepsilon_{\mathrm{NS}})T_k
	+
	2G_1\rho_{k+1}
	+
	L_R\eta^2(1+\varepsilon_{\mathrm{NS}})^2
	+
	L_R\rho_{k+1}^2 .
	\]
	Taking expectations, summing from $k=0$ to $K-1$, and using
	$F(\Gamma_K)\geq F^*$ gives
	\[
	\frac{1}{K}\sum_{k=0}^{K-1}
	\mathbb{E}\|M_{k+1}\|_1
	\leq
	\frac{D}{\eta (1-\varepsilon_{\mathrm{NS}})K}
	+
	\frac{1+\varepsilon_{\mathrm{NS}}}{1-\varepsilon_{\mathrm{NS}}}
	\frac{1}{K}\sum_{k=0}^{K-1}\mathbb{E}T_k
	+
	\frac{2G_1\bar\rho_K+L_R\bar q_K}
	{\eta (1-\varepsilon_{\mathrm{NS}})}
	+
	\frac{L_R(1+\varepsilon_{\mathrm{NS}})^2\eta}{1-\varepsilon_{\mathrm{NS}}}.
	\]
	The desired stationarity measure is related to the momentum by the triangle
	inequality:
	\[
	\|H_k\|_1
	\leq
	\|M_{k+1}\|_1+T_k .
	\]
	Combining the last two inequalities gives
	\[
	\begin{aligned}
		\frac{1}{K}\sum_{k=0}^{K-1}
		\mathbb{E}\|H_k\|_1
		\leq\;&
		\frac{D}{\eta (1-\varepsilon_{\mathrm{NS}})K}
		+
		\frac{L_R(1+\varepsilon_{\mathrm{NS}})^2\eta}{1-\varepsilon_{\mathrm{NS}}}
		+
		\frac{2G_1\bar\rho_K+L_R\bar q_K}
		{\eta (1-\varepsilon_{\mathrm{NS}})}
		\\
		&+
		\left(1+\frac{1+\varepsilon_{\mathrm{NS}}}{1-\varepsilon_{\mathrm{NS}}}\right)
		\frac{1}{K}\sum_{k=0}^{K-1}\mathbb{E}T_k .
	\end{aligned}
	\]
	Lemma~\ref{lem:momentum-tracking}, together with
	$\|\widetilde{M}_{k}\|_\infty\leq (1+\varepsilon_{\mathrm{NS}})$, gives
	\[
	\frac{1}{K}\sum_{k=0}^{K-1}\mathbb{E}T_k
	\leq
	\frac{\mathbb{E}\|M_1-H_0\|_1}{(1-\beta)K}
	+
	\frac{\sqrt{n}\,\sigma}{\sqrt{b}}
	+
	\frac{\beta L_R}{1-\beta}
	\left(\eta (1+\varepsilon_{\mathrm{NS}})+\bar\rho_K\right).
	\]
	Substituting this estimate into the previous display proves the theorem.
\end{proof}

\section*{Acknowledgments}
Portions of this work were conducted using the advanced computing resources provided by Texas A\&M High Performance Research Computing (HPRC).

\bibliography{ms}

\end{document}